\documentclass{article} %
\usepackage[utf8]{inputenc}

\usepackage{txie}

\usepackage{framed}
\usepackage{caption}
\usepackage{svg}
\usepackage{multirow}
\usepackage{multicol}
\usepackage{hyperref}
\usepackage{array}
\usepackage{longtable}
\usepackage{pythonhighlight}

\usepackage{geometry}
\usepackage{parskip}

\let\oldparagraph\paragraph
\renewcommand{\paragraph}[1]{\oldparagraph{#1.}}

\usepackage[T1]{fontenc}
\usepackage{newtxtext}
\usepackage[scaled=.91]{helvet}
\usepackage{inconsolata}
\usepackage{algorithm}
\usepackage{algpseudocode}

\allowdisplaybreaks
\title{Direct Nash Optimization:\\ \mbox{\!\!Teaching Language Models to Self-Improve with General Preferences}}

\author{
Corby Rosset\thanks{Correspondence to \texttt{\{\href{mailto:corbyrosset@microsoft.com}{corbyrosset},\href{mailto:hassanam@microsoft.com}{hassanam},\href{mailto:tengyangxie@microsoft.com}{tengyangxie}\}@microsoft.com}}
\quad
Ching-An Cheng
\quad
Arindam Mitra
\quad
Michael Santacroce \\
Ahmed Awadallah\footnotemark[1]
\quad
Tengyang Xie\footnotemark[1]
\\
\\
Microsoft Research
}

\date{}

\begin{document}

\maketitle

\algdef{SE}[SUBALG]{Indent}{EndIndent}{}{\algorithmicend\ }%
\algtext*{Indent}
\algtext*{EndIndent}

\begin{abstract}
This paper studies post-training large language models (LLMs) using preference feedback from a powerful oracle to help a model iteratively improve over itself.
The typical approach for post-training LLMs involves Reinforcement Learning from Human Feedback (RLHF), which traditionally separates reward learning and subsequent policy optimization. %
However, such a reward maximization approach is limited by %
the nature of ``point-wise'' rewards (such as that of the Bradley-Terry model), which fails to express complex intransitive or cyclic preference relations. While  advances on RLHF
show reward learning and policy optimization can be merged into a single contrastive objective
for stability, they yet still remain tethered to the reward maximization framework. %
Recently, a new wave of research sidesteps %
the reward maximization presumptions in favor of directly optimizing over ``pair-wise'' or \mbox{\textbf{general preferences}}. 
In this paper, we introduce \mbox{\emph{\DNOfull}~(\DNO)}, a \emph{provable} and \emph{scalable} algorithm that marries the \emph{simplicity} and \emph{stability} of contrastive learning with \emph{theoretical generality} from optimizing general preferences.
Because \DNO is a \emph{batched on-policy} algorithm using a regression-based objective, its implementation is straightforward and efficient. Moreover, \DNO enjoys \emph{monotonic improvement} across iterations which helps it improve even over a strong teacher (such as GPT-4). 
In our experiments, a resulting 7B parameter \orca model aligned by \DNO achieves the state-of-the-art win-rate against \gptturbo of 33\% on \alpaca (even after controlling for response length), an absolute gain of 26\% ($7\%\! \to\! 33\%$) over the initializing model. It outperforms models with far more parameters, including \mistralL, Self-Rewarding LM (70B parameters), and older versions of GPT-4.
Our ablation studies analyze critical design decisions surrounding the choice of preference pairs, and the use of LLMs-as-preference-annotators.
These results underscore the promise of \DNO in the LLMs post-training, as well as offer actionable insights for the AI research community.
\end{abstract}

\section{Introduction}

\begin{figure}[t]
    \centering
    \includegraphics[width=1.0\textwidth]{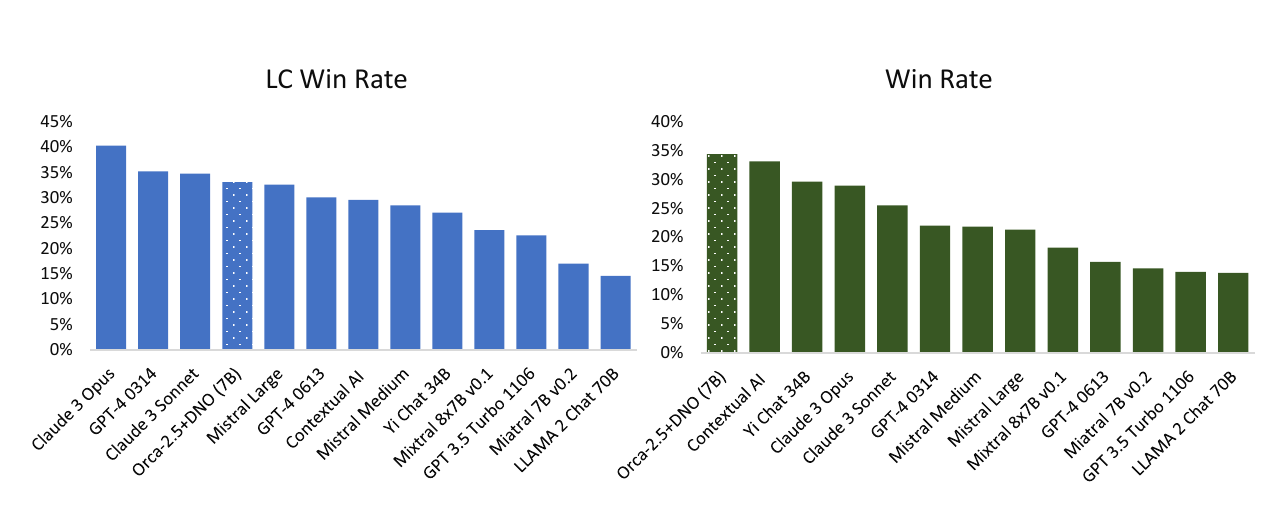}
    \caption{\DNOfull achieves state-of-the-art results for a 7B parameter large language model, being the first to surpass 30\%  in both raw win-rate and length-controlled (LC) win-rate against \gptturbo. Win Rate and LC Win Rate have $0.93$ to $0.98$ correlation with ChatBot Arena scores.}
    \label{fig:alpaca-2-bar-plot}
\end{figure}

The field of artificial intelligence is evolving towards advanced models that can understand, reason, follow complex instructions, and create nuanced content, while aligning with human values and preferences.
Large Language Models (LLMs) \citep[e.g.,][]{brown2020language,ouyang2022training,touvron2023llama,achiam2023gpt} have demonstrated remarkable capabilities in generating human-like text, answering questions, and coding, yet they still face challenges in tasks that require a high degree of reliability, safety, and ethical alignment.
To address these challenges, fine-tuning LLMs using Reinforcement Learning from Human Feedback (RLHF) \citep{christiano2017deep,bai2022training,ouyang2022training} has demonstrates strong potential for making LLMs more helpful by aligning them with human values.

The RLHF framework has been long studied in the context of  preference-based reinforcement learning (RL) or RL from human preferences \citep[e.g.,][]{knox2008tamer,akrour2012april,griffith2013policy,wirth2017survey,christiano2017deep}.
The conventional methods for RLHF typically assume that the preference is determined by a  scalar reward function through some  model, 
such as the frequently used Bradley-Terry (BT) model~\citep{bradley1952rank}.\footnote{We use ``reward model'' to denote a framework that translates preferences into rewards, e.g., Bradley-Terry, while ``reward function'' is a (possibly learned) function that outputs reward scalars.}
RLHF then optimizes toward the preference in a two-step procedure: reward learning, and policy optimization (through RL) to maximize the  learned reward. Under certain conditions, the two-step procedure can be streamlined into a single-step contrastive learning approach~\citep{rafailov2023direct}, eliminating the need for explicit reward learning. Algorithms of this kind \citep[e.g.,][DPO]{rafailov2023direct} leverage the insight that a policy can be expressed equivalently by an ``internal reward function'' that the policy is optimal to, so they reduce the RLHF problem to regressing the policy's internal reward function to that of the preference model. These algorithms are originally offline, and boast enhanced stability and ease of optimization. Nonetheless, two-step RLHF algorithms and their single-step contrastive variants still fundamentally rely on the reward maximization framework, wherein reward-based preferences are governed by, e.g., the BT model.

The reward maximization framing poses a major limitation. Reward functions, defined to output a scalar score $r(x, y)$ for a \emph{single} response $y$ to input $x$, cannot express general preferences $y \succ y' \mid x$ between a \emph{pair} of outputs in all cases, e.g., intransitive or cyclic preferences~\citep{elo1978rating}. Hence, LLMs trained under reward maximization cannot always align with human preference. Furthermore, recent works show that even in settings where preferences can be perfectly expressed under the reward-based BT models, optimizing towards rewards yields problematic behaviors; we refer the reader to~\cite{bertrand2023limitations,azar2023general,munos2023nash} for more details.
Lastly, reward functions in practice can quickly become ``stale'' as the distribution of the policy shifts under training~\citep{ross2011reduction,cheng2023adversarial,azar2023general,munos2023nash} -- leaving them vulnerable to ``reward hacking''~\citep{amodei2016concrete}%

In response to these weaknesses, an appealing line of work on RLHF proposes to directly optimize the \emph{general preference function} itself, instantiated as some oracle.
These studies re-frame RLHF as finding a Nash equilibrium of a two-player game with ``payoffs'' 
from a regularized~\citep{munos2023nash} or un-regularized~\citep{swamy2024minimaximalist} general preference function. 
To solve it, they further approximate such Nash equilibrium using single-player algorithms by leveraging the symmetry of the preference function.
Then, they instead define the reward of a response as the expected win-rate against the policy's own behavior, as judged by the preference function, e.g., $ r(x, y) = \E_{y' \sim \pi(\cdot \mid x)}\left[\Pcal(y \succ y' \mid x)\right]$. 
Hence, rewards are maximized by responses that are preferred over the policy's expected response, and a Nash equilibrium is achieved when both players deploy a $\pi^\star$ that is preferred over any competing policy. 
However, these proposed single-player algorithms are primarily (purely) on-policy, and they sometimes require a separately estimated preference function or a time-varying reward function. How to scale these algorithms up faithfully is still under-investigated.

We are motivated to overcome two separate challenges: the limited expressivity of reward-based RLHF, and the lack of clarity on how to scale up optimizing with respect to general preferences. Recent advances in reward-based optimization e.g., DPO, already have efficient and scalable implementations -- we seek a similarly efficient solution under the framework of general preferences. 

We propose a provable and scalable RLHF algorithm -- {\bf\em\DNOfull}~(\DNO) (\cref{alg:dno}) that achieves the best of both worlds, combining the scalability of contrastive objectives with the theoretical soundness of general preference optimization. \DNO is designed as a \emph{batched on-policy} algorithm with a regression-based learning objective; this design choice makes \DNO stable and scalable, striking  a balance between deployment efficiency and adaptability.

We summarize at a high level the key ingredients and insights of \DNO below.
\begin{enumerate}[nolistsep]
    \item To address the issue that reward functions cannot express general preferences, we leverage recent insights that the notion of reward of ought to be expressed as \emph{expected} win-rates with regard to a general preference function.\footnote{E.g., for a fixed $y \mid x$, the expected win-rate of $y$ against the policy itself is: $\E_{y' \sim \pi(\cdot \mid x)}\left[\Pcal(y \succ y' \mid x)\right]$.}
    \item To address the issue found in previous work that optimizing this more general objective with online algorithms is sample-inefficient or unstable, we decompose the learning procedure into a sequence of ``batched on-policy'' iterations, wherein each step instead optimizes a simple regression objective. 
    \item The regression objective (we choose binary cross-entropy) aligns the ``internal reward function'' of the policy to the expected win-rate compared with itself (as defined in \cref{step:def_rt} of \cref{alg:dno}). By sampling outputs from the current policy to use for training (i.e., ``self-play''), this procedure incentivizes self-improving behavior. 
    \item Our framework is general enough to admit off-policy samples into training, importantly, those from a more powerful teacher (See choice of $\mu_1$ and $\mu_2$ in ~\cref{alg:dno}).    
    \item Furthermore, to ensure stability and computational efficiency, we propose a filtering scheme such that the reward regression is only performed on preference pairs with a sufficiently large margin (for theoretical explanation, see \cref{paragraph:preference-pair-construction}; in practice, see \cref{sec:results}).
    \item \DNO repeats this procedure for multiple iterations to let the policy optimize toward the general preference. Since each step involves a regression problem it can be easily implemented at scale.

\end{enumerate}

Theoretically, we prove \DNO converges to the intended Nash equilibrium on average, and that  it can improve monotonically across iterations (see \cref{sec:monotonic-improvement}). Furthermore, our finite-sample analysis shows that approximation error at any iteration between the learned policy and the target is tightly bounded (\cref{thm:dno_guarantee_main}).

On the practical side, we provide a scalable implementation of \DNO (\cref{alg:NashDPO}): an iterative self-improving algorithm with contrastive updates, which approximates \cref{alg:dno} under several critical design choices. Those choices include: sampling multiple online outputs from the policy being trained, using GPT-4 as the preference oracle, comparing on-policy samples to GPT-4's own (teacher) outputs, and training only on pairs with ``large margin'' (for theoretical explanation, see \cref{paragraph:preference-pair-construction}; in practice, see \cref{sec:results}). 

The primary distinction of our work over related works of Nash-MD~\citep{munos2023nash} and SPO~\citep{swamy2024minimaximalist} is that they both exhibit sample efficiency issues (two timescale updates or sample-inefficient RL steps), and both use purely on-policy samples. We resolve the efficiency issue with a sample-efficient objective that works in practice, and \DNO is more flexible to incorporate off-policy samples from e.g., a powerful teacher.

Most importantly, \DNO works in practice -- we provide comprehensive empirical evaluations, resulting in state-of-the-art performance:
\begin{itemize}[nolistsep]
\item The resulting 7B parameter \orca model, aligned using the practical implementation of \DNO (\cref{alg:NashDPO}), achieves the state-of-the-art win-rate of any 7B model, exceeding $33\%$ against \gptturbo beyond on the \alpaca, even after controlling for length. This is an over $26\%$ absolute gain ($7\%\!\to\! 33\%$) compared to the initialized model. It outperforms several recent advanced closed-source models, including \mistralL and \gptold, as well as open-source models with far more ($10\times$) parameters, such as Self-Rewarding LM \citep{yuan2024selfrewarding} which has 70B parameters.
\item Our thorough ablation studies in~\cref{sec:results} examine critical design touchpoints surrounding choice of loss function (supervised finetuning or contrastive), training paradigm (with or without on-policy samples), preference annotator quality (large margin or not), and training pair construction (self-play, teacher-vs-student, etc). Our findings highlight that carefully-crafted methods encoded in \cref{alg:NashDPO} lead to substantial gains.
\item We show some examples of outputs across iterations which demonstrate qualitative improvements such as better addressing nuanced issues and presumptious questions (\cref{ex:example-1}), better organization and clarity while refraining from making misleading statements (\cref{ex:example-2}), and higher information density in answers (\cref{ex:example-3}).  
\end{itemize}
We hope that the results presented herein will provide clarity to the community regarding the use of AI feedback for post-training LLMs.

\section{Preliminaries}

This section provides an overview of the RL from human feedback (RLHF) pipeline. We do not differentiate between RLHF and RLAIF \citep[e.g.,][]{bai2022constitutional,lee2023rlaif}, as the distinction is outside our scope of discussion. Thus, we will uniformly refer to both concepts as RLHF. However, we want to make a clear delineation between two subtle differences: RLHF maximizing point-wise reward functions, and RLHF optimizing general preferences. It should be noted that this discussion is more broadly applicable in scope to general contextual bandits setup as well.

Throughout this paper, we use $x \in \Xcal$ to denote the input (i.e. the prompt) received by the LLM from a space $\Xcal$.
In this paper, we do not consider the distribution shift over the prompts, following the standard contextual bandits setup of RLHF \citep[e.g.,][]{ouyang2022training,rafailov2023direct}, and we use $\rho$ to denote the distribution of the prompts.
We use $y \in \Ycal$ to denote the response from the LLM given the prompt $x$ (this corresponds to action in the contextual bandits setup).
We also use $\pi:\Xcal \to \Delta(\Ycal)$ to denote the policy, which is a LLM here, and $\Pi$ is the policy class.

Our discussion throughout this paper will also regularly involve the following three learning paradigms, which are originally introduced and commonly used in the RL literature:
\begin{enumerate}[(1), nolistsep]
    \item \emph{Offline}: The learning algorithm operates without any active data collection, e.g., sampling from the current policy. The algorithm relies solely on an offline dataset for training.
    \item \emph{Purely on-policy}: technically,  online on-policy. In this setup, learning takes place by sampling outputs from the latest policy and immediately updating it based on the newly collected data. No data reuse or additional offline data is considered.
    \item \emph{Batched on-policy}\footnote{
    We acknowledge abuse of terminology.
    Our algorithm is not entirely online, as it only contains batched data collection.
    It is also not strictly on-policy because it uses examples from other policies, like a teacher. While ``offline'' or ``off-policy'' may be technically more relevant, they might lead to misunderstanding among readers and detract from the emphasis we want to place on the collection samples from the current policy, which constitute the majority of our training data.}:
    This setup is the middle of the offline and purely on-policy setups, striking a balance between deployment efficiency and adaptability.
    It involves iterative online data collection and can use other offline data.
    Its distinctive feature is that here the data collection in each iteration occurs in a batched fashion (e.g., akin to a dataset scale, much larger than the size of a typical mini-batch), %
    and the amount of policy change can be more significant (e.g., running gradient steps over multiple epochs of a dataset, as opposed to tens of updates).
\end{enumerate}

\subsection{RLHF Based on Reward Models}
One typical approach to conducting RLHF is a two-step procedure through a reward function \citep{christiano2017deep}.
Suppose a preference dataset $\Dcal_\pref \coloneqq \{(x,y^+,y^-)\}$ is given, where $(y^+,y^-) \sim \piref(\cdot \mid x)$, $\piref$ is some reference policy such as the policy obtained after supervised fine-tuning (SFT), and a preference $y^+ \succ y^- \mid x$ is labeled by some human or AI annotator.
In RLHF with reward functions, the preference is assumed to be generated based on some latent reward  $r^\star$.
The first step is to learn a reward function $r \in \Rcal$ under some reward model assumption, where $\Rcal$ is the reward class.
A number of reward model assumptions have been studied, and the Bradley-Terry (BT)  model \citep{bradley1952rank} is the most commonly used one.
The BT model assumes the probability of $y^+ \succ y^- \mid x$ satisfies
\begin{align*}
\Pcal(y^+ \succ y^- \mid x) \coloneqq \frac{\exp(r^\star(x,y^+))}{\exp(r^\star(x,y^+)) + \exp(r^\star(x,y^-))}.
\end{align*}
This leads to the maximum-likelihood reward learning objective:
\begin{align}
\label{eq:ppo_rm}
\widehat r \leftarrow \argmax_{r \in \Rcal} \E_{(x,y^+,y^-) \sim \Dcal}\left[\log\sigma(r(x,y^+) - r(x,y^-)) \right],
\end{align}
where $\sigma(\cdot) \coloneqq \frac{\exp(\cdot)}{1 + \exp(\cdot)}$ is the sigmoid function. After that, the LLM is finetuned using the learned $\rhat$ with RL,
\begin{align}
\label{eq:ppo_pg}
\pihat \leftarrow \argmax_{\pi \in \Pi} \mathbb{E}_{x \sim \Dcal_\pref, y \sim \pi(\cdot \mid x)} \left[\rhat(x, y) - \beta \log\frac{\pi(x, y)}{\piref(x, y)}\right],
\end{align}
where the KL penalty term, $\mathbb{E}_{x \sim \Dcal_\pref, y \sim \pi(\cdot \mid x)} \left[\beta \log\frac{\pi(x, y)}{\piref(x, y)}\right]$, is used to mitigate overoptimization of the reward model~\citep{ouyang2022training}, controlled by the coeffcient $\beta$. For the purposes of our discussion, we will call the objective above as ``\PPO objective'', and this two-step learning procedure as ``\PPO''.

\paragraph{\DPO}

Direct Preference Optimization (\DPO) is proposed by \citet{rafailov2023direct} as an alternative RLHF approach for combining the two-step procedure of \PPO into a single objective. It utilizes the closed form solution $\pihat$ in \cref{eq:ppo_pg}, so that solving $\pihat$ directly from \cref{eq:ppo_rm} becomes possible via
\begin{align*}
\pihat \leftarrow \argmax_{\pi \in \Pi} \E_{(x,y^+,y^-) \sim \Dcal_\pref} \log \left[ \sigma\left( \beta \log\frac{\pi(y^+ \mid x)}{\piref(y^+ \mid x)} - \beta \log\frac{\pi(y^- \mid x)}{\piref(y^- \mid x)} \right)\right].
\end{align*}

\subsection{RLHF with General Preferences}

We now introduce the setup for directly optimizing a general preference function, as well as provide an overview of existing solutions to achieve this goal (mostly by leveraging the symmetry of the preferences), especially those proposed by \citet{munos2023nash,swamy2024minimaximalist}.

Here we assume that the learner is given query access to a general preference function
$\Pcal(y \succ y' \mid x) \in [0,1]$, for any $(x,y,y') \in \Xcal \times \Ycal \times \Ycal$.
This function indicates the probability that action $y$ is preferred over $y'$ given the context $x$.
In practice, this setup can be viewed as the theoretical mode of RLAIF \citep[e.g.,][]{bai2022constitutional,yuan2024selfrewarding}, human-in-the-loop RLHF \citep[e.g.,][]{ouyang2022training}, or distillation fine-tuning \citep[e.g.,][]{tunstall2023zephyr}.

One common difficulty in optimizing a general preference function is its \emph{intransitivity}, e.g., it is possible that $\Pcal(a \succ b) = \Pcal(b \succ c) = \Pcal(c \succ a) = 1$, for some options $(a,b,c)$ \citep[details see, e.g.,][]{bertrand2023limitations, munos2023nash,swamy2024minimaximalist}. Therefore, the learning goal of optimizing general preferences can be the Nash equilibrium of the two-player zero-sum game with the payoffs as the general preference function $\Pcal$. The formal definition of such Nash equilibrium is defined by the \emph{Minimax Winner}, $\MW$~\citep[see, e.g.,][]{kreweras1965aggregation,simpson1969defining,kramer1973class,fishburn1984probabilistic}, or the \emph{von Neumann Winner} \citep[see, e.g.,][]{dudik2015contextual},
\begin{align}
\label{eq:defMW}
\MW(\Pcal) \coloneqq \argmax_{\pi \in \Pi} \argmin_{\pi' \in \Pi} \Pcal(\pi \succ \pi') = \left(\argmax_{\pi \in \Pi} \min_{\pi' \in \Pi} \Pcal(\pi \succ \pi'), \argmin_{\pi' \in \Pi} \max_{\pi \in \Pi} \Pcal(\pi \succ \pi') \right),
\end{align}
where
\begin{align*}
\Pcal(\pi \succ \pi') \coloneqq \E_{x \sim \rho, y \sim \pi(\cdot\mid x), y' \sim \pi'(\cdot\mid x)} \left[ \Pcal(y \succ y' \mid x) \right].
\end{align*}

\paragraph{\SPO}
To approximate the Nash equilibrium as defined in \cref{eq:defMW}, \citet{swamy2024minimaximalist} proposed a single-player algorithm, \SPO. This algorithm applies results from no-regret algorithms \citep[e.g.,][]{freund1997decision}. The \SPO algorithm is executed essentially using the following two-step iterative process: for each $t = 1,2,\dotsc,T$,
\begin{align*}
    \text{(i)} \quad &~ r_t(x,y) \leftarrow \E_{y' \sim \pi_t(\cdot \mid x)}\left[\Pcal(y \succ y' \mid x)\right],~~\forall (x,y) \in \Xcal \times \Ycal
    \\
    \label{eq:spo_spi}\numberthis
    \text{(ii)} \quad &~ \pi_{t+1}(\cdot \mid x) \leftarrow \frac{1}{Z_t(x)} \pi_t(\cdot \mid x) \exp\left(\frac{r_t(x,\cdot)}{\eta}\right),~~\forall x \in \Xcal,
\end{align*}
where $\eta$ is the learning rate, $\pi_1$ is the uniform policy, i.e., $\pi_1(\cdot \mid x) \leftarrow \unif(\Ycal),~\forall x \in \Xcal$, and $Z_t(x) \coloneqq \sum_{y \in \Ycal} \pi_t(y \mid x) \exp\left(\frac{r_t(x,y)}{\eta}\right)$ is the partition function for iteration $t$.

Using the no-regret update of soft policy iteration, as shown in \cref{eq:spo_spi}, \citet{swamy2024minimaximalist} proved that the uniform mixture of $\pi_{1:T}$ from \SPO is an approximation of the Nash equilibrium of $\MW(\Pcal)$, as defined in \cref{eq:defMW}.

\paragraph{\NashMD}

\citet{munos2023nash} proposed \NashMD to approximate the Nash equilibrium of a KL-regularized preference function,
\begin{align}
\label{eq:def_reg_pref}
\Pcal^{\pi,\pi'}_{\tau}(y \succ y' \mid x) \coloneqq &~ \Pcal(y \succ y' \mid x) - \tau \log\frac{\pi(y \mid x)}{\piref(y \mid x)} + \tau \log\frac{\pi'(y \mid x)}{\piref(y \mid x)},
\\
\Pcal_\tau(\pi \succ \pi') \coloneqq &~ \E_{x \sim \rho, y \sim \pi(\cdot\mid x), y' \sim \pi'(\cdot\mid x)} \left[ \Pcal^{\pi,\pi'}_{\tau}(y \succ y' \mid x) \right]
\\
= &~ \Pcal(\pi \succ \pi') - \tau \E_{x \sim \rho}\left[ \KL(\pi(\cdot \mid x) \midd \piref(\cdot \mid x)) \right] + \tau \E_{x \sim \rho}\left[ \KL(\pi'(\cdot \mid x) \midd \piref(\cdot \mid x)) \right]. \nonumber
\end{align}
Following this, \citet{munos2023nash} demonstrate that the Nash Equilibrium of $\MW(\Pcal_\tau)$ can be approximated using a mirror descent \citep{nemirovskij1983problem,bubeck2015convex,lattimore2020bandit} inspired algorithm, \NashMD, which has a last-iteration guarantee. The \NashMD algorithm can be viewed as a two-step iterative process: for each $t = 1,2,\dotsc,T$,
\begin{align*}
    \text{(i)} \quad &~ r_t(x,y) \leftarrow \E_{y' \sim \pi^{\tau}_t(\cdot \mid x)}\left[\Pcal(y \succ y' \mid x)\right],~~\forall (x,y) \in \Xcal \times \Ycal
    \\
    \text{(ii)} \quad &~ \pi_{t+1}(\cdot \mid x) \leftarrow \frac{1}{Z_t(x)} \pi_t^\tau(\cdot \mid x) \exp\left(\frac{r_t(x,\cdot)}{\eta}\right),~~\forall x \in \Xcal,
\end{align*}
where $\eta$ is the learning rate, $\pi_t^\tau$ is the geometric mixture between $\pi_t$ and $\piref$,
\begin{align}
\label{eq:def_smooth_pit}
\pi_t^\tau(y \mid x) \coloneqq \frac{\pi_t(y \mid x)^{1 - \nicefrac{\tau}{\eta}}\piref(y \mid x)^{\nicefrac{\tau}{\eta}}}{\sum_{y' \in \Ycal}\pi_t(y \mid x)^{1 - \nicefrac{\tau}{\eta}}\piref(y \mid x)^{\nicefrac{\tau}{\eta}}}, ~ \forall (x,y) \in \Xcal\times\Ycal,
\end{align}
and $Z_t(x) \coloneqq \sum_{y \in \Ycal} \pi_t^\tau(y \mid x) \exp\left(\frac{r_t(x,y)}{\eta}\right)$ is the partition function for iteration $t$. 

\section{\DNOfull}
\label{sec:dno}

While the no-regret update of soft policy iteration used in \SPO and \NashMD has inspired many standard (deep) reinforcement learning algorithms (e.g., \citealp[NPG,][]{kakade2001natural}; \citealp[TRPO,][]{schulman2015trust}; \citealp[PPO,][]{schulman2017proximal}; \citealp[SAC,][]{haarnoja2018soft}), its faithful implementation still usually involves the two-timescale update. This could potentially lead to complex hyperparameter tuning and unstable performance.
In this section, we propose a direct and iterative algorithm, {\em \DNOfull} (\cref{alg:dno}), to approximate the Nash equilibrium of $\MW(\Pcal)$. This algorithm is primarily inspired by \SPO. 
It can be readily adapted to \NashMD for approximating the Nash equilibrium of $\MW(\mathcal{P}_\tau)$ with the last-iteration guarantee, and we will discuss this in \cref{app:nash_md_ver}.

\begin{algorithm}[ht]
\caption{\DNOfull (\DNO)}
\label{alg:dno}
{\bfseries input:} General preference function $\Pcal$, learning rate $\eta$, number of iterations $T$, prompt distribution $\rho$.
\begin{algorithmic}[1]
\State Initialize $\pi_1 \leftarrow \unif(\Acal)$.
\For{iteration $t = 1,2,\dotsc,T$}
    \State\label{step:def_rt} Compute $r_t(x,y) \leftarrow \E_{y' \sim \pi_t(\cdot \mid x)}\left[\Pcal(y \succ y' \mid x)\right]$, ~$\forall (x,y)\in\Xcal\times\Ycal$.
    \State\label{step:dno_learnpi} Obtain $\pi_{t+1}$ by,
    \begin{equation}
    \label{eq:dnoloss}
    \begin{gathered}
    \pi_{t+1} \leftarrow \argmax_{\pi \in \Pi} \E_{(x,y_1,y_2) \sim \Dcal_t}\bigg\{ \sigma\left( r_t(x,y_1) - r_t(x,y_2) \right) \log \left[ \sigma\left( \eta \log\frac{\pi(y_1 \mid x)}{\pi_t(y_1 \mid x)} - \eta \log\frac{\pi(y_2 \mid x)}{\pi_t(y_2 \mid x)} \right)\right]
    \\
    \hspace{120pt} + \sigma\left( r_t(x,y_2) - r_t(x,y_1) \right) \log \left[ \sigma\left( \eta \log\frac{\pi(y_2 \mid x)}{\pi_t(y_2 \mid x)} - \eta \log\frac{\pi(y_1 \mid x)}{\pi_t(y_1 \mid x)} \right) \right] \bigg\},
    \end{gathered}
    \end{equation}
    where $\Dcal_t$ is generated by $x \sim \rho, y_1 \sim \mu_{1,t}(\cdot \mid x), y_2 \sim \mu_{2,t}(\cdot \mid x)$; $\mu_{1,t}$ and $\mu_{2,t}$ can be either off-policy (e.g., pre-defined) or on-policy (based on $\pi_t$).
\EndFor
\State \Return $\pibar = \unif(\pi_{1:T})$.
\end{algorithmic}
\end{algorithm}

\subsection{Derivation of \cref{alg:dno}}
\label{sec:derivation}
In most practical algorithms which are inspired by soft policy iteration, including the original practical version of \SPO, they typically adopt the following approach: ``pushing'' $\pi$ towards this subsequent learning goal in each iteration (we will refer to this as the soft policy iteration target throughout the paper):
\begin{align}
\label{eq:def_pistar}
\pi^\star_{t+1}(\cdot \mid x) \coloneqq \frac{1}{Z_t(x)} \pi_{t}(\cdot \mid x) \exp\left(\frac{r_t(x,\cdot)}{\eta}\right),
\end{align}
where $Z_t(x) = \sum_{y \in \Ycal} \pi_t(y \mid x) \exp\left(\frac{r_t(x,y)}{\eta}\right)$ is the partition function. 
It can be realized by minimizing a distance metric between $\pi_{t+1}$ and $\pi$. For example, the \PPO algorithm for RLHF \citep[e.g.,][]{christiano2017deep,ouyang2022training} essentially minimizes the reverse KL divergence as follows,
\begin{align*}
(\pi_{t+1}^\text{PPO} \leftarrow)
&~ \argmin_{\pi \in \Pi} \E_{x \sim \rho}\left[ \KL(\pi(\cdot \mid x) \midd \pi^\star_{t+1}(\cdot \mid x)) \right]
\\
= &~ \argmax_{\pi \in \Pi} \mathbb{E}_{x \sim \rho, y \sim \pi(\cdot \mid x)} \left[\eta\log\frac{\pi^\star_{t+1}(x, y)}{\pi_t(x, y)} - \eta \log\frac{\pi(x, y)}{\pi_t(x, y)}\right]
\\
= &~ \argmax_{\pi \in \Pi} \mathbb{E}_{x \sim \rho, y \sim \pi(\cdot \mid x)} \left[r_t(x, y) - \eta Z_t(x) - \eta \log\frac{\pi(x, y)}{\pi_t(x, y)}\right]
\\
= &~ \argmax_{\pi \in \Pi} \mathbb{E}_{x \sim \rho, y \sim \pi(\cdot \mid x)} \left[r_t(x, y) - \eta \log\frac{\pi(x, y)}{\pi_t(x, y)}\right].
\tag{$\Leftrightarrow$ \PPO objective, as $Z_t$ is independent of $\pi$}
\end{align*}
However, implementing the above approach typically necessitates \emph{on-policy} sampling from the current policy $\pi$. Ignoring the $Z_t(x)$ term could also lead to high variance in the empirical gradient estimation. This is a persistent issue in actor-critic style algorithms that usually suggests the need for an additional baseline \citep[details see, e.g.,][]{mnih2016asynchronous}, which also requires on-policy estimation.
When $r_t$ also varies over iterations, as in \SPO or \NashMD, we then need to update all of the policy, baseline, and reward online simultaneously.
These challenges have hindered the scalability of existing algorithms which are based on learning the Nash equilibrium of general preference functions.
\paragraph{Regressing ``internal rewards'' towards preference-based rewards}
Different from the mentioned approaches above which are mostly focusing on the concept of ``pushing'' $\pi \to \pi^\star_{t+1}$. We now consider the following mechanism: {\em regressing} $r_{\pi,t} \to r_t$, where $r_{\pi,t}$ is the internal reward function of a given $\pi$ at iteration $t$:
\begin{align}
\label{eq:def_rpit}
r_{\pi,t}(x,y) \coloneqq \eta \log\frac{\pi(y \mid x)}{\pi_t(y\mid x)} + \eta Z_t(x).
\end{align}
This can be interpreted as a reparameterization trick, where $\pi$ is exactly the soft policy iteration target (refer to \cref{eq:def_pistar}) induced by $\pi_t$ and the defined $r_{\pi,t}$.
Therefore, regressing that specifically parameterized $r_{\pi,t}$ to $r_t$ allows us to directly optimize the soft policy iteration target with respect to $r_{\pi,t}$ and $\pi_t$. This idea is inspired by techniques from inverse RL \citep[e.g.,][Guided Cost Learning]{finn2016guided,finn2016connection} as well as recent advances in RLHF \citep[DPO]{rafailov2023direct}.
To avoid the issues arising from the partition function $Z_t(x)$, we consider learning from the $(x, y_1, y_2)$ tuple, where $y_1$ and $y_2$ are both responses to textual input $x$.
Note that, due to the offline learning nature of the regressive objective, the sampling distribution of $y_1$ and $y_2$ does not impact the learning objective (i.e., $r_{\pi,t} \to r_t$, but it may affect the sample complexity from the coverage reason as we will discuss later), whereas pushing $\pi \to \pi_{t+1}^\star$ requires sampling $y$ on-policy, as previously discussed.
Therefore, given an arbitrary $(x,y_1,y_2)$ tuple, we regress the ``prediction'' $\hat{z}$ to the ``goal'' $z$ (both defined below), using binary logarithmic/cross-entropy loss to measure the prediction error \citep[see, e.g.,][]{foster2021efficient},
\begin{small}
\begin{gather*}
\label{eq:def_z}\numberthis
\zhat \coloneqq \sigma \left( r_{\pi,t}(x,y_1) - r_{\pi,t}(x,y_2) \right) = \sigma \bigg( \underbrace{\eta \log\frac{\pi(y_1 \mid x)}{\pi_t(y_1 \mid x)} - \eta \log\frac{\pi(y_2 \mid x)}{\pi_t(y_2 \mid x)}}_{\eqqcolon \Delta_{\pi,t}(x,y_1,y_2)} \bigg),\quad z \coloneqq \sigma\big(\underbrace{r_t(x,y_1) - r_t(x,y_2)}_{\eqqcolon \Delta_t^\star(x,y_1,y_2)}\big);
\\
\ell_{\pi,t}(x,y_1,y_2) \coloneqq z \log(1/\zhat) + (1 - z) \log(1/(1-\zhat))
\\
= - \sigma\big( \Delta_t^\star(x,y_1,y_2) \big) \log \left[ \sigma\big( \Delta_{\pi,t}(x,y_1,y_2) \big) \right] - \sigma\big( \Delta_t^\star(x,y_2,y_1) \big) \log \left[ \sigma\big( \Delta_{\pi,t}(x,y_2,y_1) \big) \right].
\end{gather*}
\end{small}
Therefore, we obtain the following objective to learn $\pi_{t+1}$,
\begin{align*}
&~ \argmin_{\pi \in \Pi} \Lcal_{\Dcal_t} (\pi; \pi_t) 
\\
\coloneqq &~ \argmin_{\pi \in \Pi} \E_{(x,y_1,y_2) \sim \Dcal_t}\left[ \ell_{\pi,t}(x,y_1,y_2) \right]
\\
= &~ \argmax_{\pi \in \Pi} \E_{(x,y_1,y_2) \sim \Dcal_t}\Big[ \sigma\big( \Delta_t^\star(x,y_1,y_2) \big) \log \left[ \sigma\big( \Delta_{\pi,t}(x,y_1,y_2) \big) \right] + \sigma\big( \Delta_t^\star(x,y_2,y_1) \big) \log \left[ \sigma\big( \Delta_{\pi,t}(x,y_2,y_1) \big) \right] \Big]
\\
\label{eq:reward_regression}\numberthis
= &~ \argmax_{\pi \in \Pi} \E_{(x,y_1,y_2) \sim \Dcal_t}\Bigg[ \sigma\left( r_t(x,y_1) - r_t(x,y_2) \right) \log \left[ \sigma\left( \eta \log\frac{\pi(y_1 \mid x)}{\pi_t(y_1 \mid x)} - \eta \log\frac{\pi(y_2 \mid x)}{\pi_t(y_2 \mid x)} \right) \right]
\\
&~ \hspace{100pt} + \sigma\left( r_t(x,y_2) - r_t(x,y_1) \right) \log \left[ \sigma\left( \eta \log\frac{\pi(y_2 \mid x)}{\pi_t(y_2 \mid x)} - \eta \log\frac{\pi(y_1 \mid x)}{\pi_t(y_1 \mid x)} \right) \right] \Bigg].
\end{align*}
Here, $\Dcal_t$ is generated by $x \sim \rho, y_1 \sim \mu_{1,t}(\cdot \mid x), y_2 \sim \mu_{2,t}(\cdot \mid x)$ with some policies $\mu_{1,t}$ and $\mu_{2,t}$. It should be noted that $\mu_{1,t}$ and $\mu_{2,t}$ for each $t \in [T]$ are parts of our algorithm's design decisions. We will provide choices for them in \cref{sec:dno_theory} to promote sample efficiency, which are informed by our finite-sample analysis.

\paragraph{Monotonic improvement from the \emph{batched on-policy} updates}
\label{sec:monotonic-improvement}
One key distinction between \DNO and existing algorithms for learning Nash equilibrium (such as \SPO and \NashMD) is that those algorithms aim to approach the Nash equilibrium in a purely on-policy manner, which can be potentially unstable and may need to incorporate two-timescale updates (that change the reward function used in the inner problem more frequently).
On the other hand, \DNO is a batched on-policy algorithm with single-timescale updates.

From a purely theoretical perspective, it seems that \DNO may require many iterations to ensure the convergence of $\pibar$ to the Nash equilibrium, which could potentially be costly.
Additionally, \DNO only converges on-average, and it is unrealistic to deploy in practice that uniform mixture policy $\pibar$ (note that, as inspired by \citet{munos2023nash}, \DNO could be extended to regularized preferences with last-iteration convergence, which is discussed in \cref{app:nash_md_ver}).
However, from a practical perspective, we can leverage the following two desirable properties from LLMs scenario to eliminate these concerns and ensure \emph{monotonic improvement} over the \DNO iterations:

Firstly, the soft policy iteration target \cref{eq:def_pistar} is actually the analytical solution for maximizing the following loss, $\ell_t(\pi) \coloneqq \Pcal(\pi \succ \pi_t) - \eta\E_{x \sim \rho}\left[ \KL(\pi(\cdot \mid x) \midd \pi_t(\cdot \mid x)) \right]$, and $\pi_{t+1}^\star = \argmax_{\pi}\ell_t(\pi)$. We can notice that $\ell_t(\pi_t) = 0.5$ and $\E_{x \sim \rho}\left[ \KL(\pi(\cdot \mid x) \midd \pi_t(\cdot \mid x)) \right] \geq 0$. This means $0.5 \leq \ell_t(\pi_{t+1}^\star) = \Pcal(\pi_{t+1}^\star \succ \pi_t) - \eta\E_{x \sim \rho}\left[ \KL(\pi_{t+1}^\star(\cdot \mid x) \midd \pi_t(\cdot \mid x)) \right] \Longrightarrow \Pcal(\pi_{t+1}^\star \succ \pi_t) \geq 0.5 + \eta\E_{x \sim \rho}\left[ \KL(\pi_{t+1}^\star(\cdot \mid x) \midd \pi_t(\cdot \mid x)) \right]$.
This means $\pi_{t+1}^\star$ is guaranteed to be more preferred than $\pi_t$ with respect to the preference $\Pcal$, and there is even a computable lower bound of the amount of improvement---$\eta\E_{x \sim \rho}\left[ \KL(\pi_{t+1}^\star(\cdot \mid x) \midd \pi_t(\cdot \mid x)) \right]$.
Therefore, if $\pi_{t+1}$ learned from \cref{step:dno_learnpi} of \cref{alg:dno} is a accurate enough approximation of $\pi_{t+1}^\star$ (which is proved in \cref{sec:dno_theory}), we could expect that the policy is monotonically improved over \DNO iterations. 
Note that the monotonic improvement guarantee is \emph{exclusive} to our design choice of \emph{batched on-policy} updates in \DNO, because the alternatives are either unclear or unstable: it is undefined how to perform iterative updates offline, and one gradient update from a purely online algorithm may not be able able to accurately approximate the the soft policy iteration target $\pi_{t+1}^\star$. Secondly, in practice, we usually have validation data available, which allows us to deploy the best policy over $\pi_{1:(T+1)}$.

\subsection{Theoretical Analysis}
\label{sec:dno_theory}

One of our major proposals is to use a regression-based objective to approximate the explicit soft policy iteration; in this section we show the approximation error from this regression is tightly bounded with finite-sample analysis. 
The following proposition discusses how well the solution of the regression-based objective (defined in \cref{eq:reward_regression} or \cref{step:dno_learnpi} of \cref{alg:dno}) can approximate the soft policy iteration (\cref{eq:def_pistar}) in terms of the total variation metric at each iteration.
\begin{theorem}[informal]
\label{thm:dno_guarantee_main}
Fix an arbitrary iteration $t \in [T]$. Suppose $\pi_{t+1}$ is from \cref{step:dno_learnpi} of \cref{alg:dno}, and $\pi_{t+1}^\star$ is defined in \cref{eq:def_pistar}. Then, under mild assumptions (realizability and boundedness, formally introduced in \cref{app:dno_proof}), we have
\begin{align*}
\E_{x \sim \rho}\left[ \left( \TV(\pi_{t+1}(\cdot \mid x),\pi_{t+1}^\star(\cdot \mid x)) \right)^2\right] \leq \Ocal \left( \frac{\Cfrak_t\Rmax^2\log(\nicefrac{|\Pi|}{\delta})}{N} \right),
\end{align*}
where the concentrability coefficient $\Cfrak_t$ is defined as below,
\begin{align*}
\Cfrak_t \coloneqq \frac{\E_{x \sim \rho, y_1 \sim \pi_{t+1}^\star(\cdot \mid x), y_2 \sim \pi_{t+1}(\cdot \mid x)}\left[ \left( \log\frac{\pi_{t+1}^\star(y_1 \mid x)}{\pi_{t+1}(y_1 \mid x)} - \log\frac{\pi_{t+1}^\star(y_2 \mid x)}{\pi_{t+1}(y_2 \mid x)} \right)^2\right]}{\E_{x \sim \rho, y_1 \sim \mu_{1,t}(\cdot \mid x), y_2 \sim \mu_{2,t}(\cdot \mid x)}\left[ \left( \log\frac{\pi_{t+1}^\star(y_1 \mid x)}{\pi_{t+1}(y_1 \mid x)} - \log\frac{\pi_{t+1}^\star(y_2 \mid x)}{\pi_{t+1}(y_2 \mid x)} \right)^2\right]}.
\end{align*}
\end{theorem}

If $\pi_t = \pi_t^\star$ for all $t \in [T]$, the reader can refer to \citep[Section 3]{swamy2024minimaximalist} for the convergence of $\pibar$ (returned by \cref{alg:dno}) to the Nash equilibrium. We expect the total variation difference between $\pi_t$ and $\pi_t^\star$ provided by \cref{thm:dno_guarantee_main} will be additive errors on top of the guarantees from \citet{swamy2024minimaximalist}.

Note that, we present the concentrability coefficient $\Cfrak_t$ as data-dependent, with $\pi_{t+1}$ (learned from data) as part of its definition. We aim to make this guiding the design choices of $\mu_{1,t}$ and $\mu_{2,t}$ from such $\Cfrak_t$ for the purpose of sample efficiency.
The formal statement and detailed proof of \cref{thm:dno_guarantee_main}, without involving $\pi_{t+1}$, are deferred to \cref{app:dno_proof}.
Although it shares a similar expression to the concentrability coefficient in offline reinforcement learning \citep[e.g.,][]{chen2019information,xie2021bellman}, the policies $\mu_{1,t}$ and $\mu_{2,t}$ are flexible here due to the generative nature of large language models. This flexibility allows for additional intervention, enhancing sample efficiency.

We can notice that the value of $\Cfrak_t$ can be always bounded by $\Cfrak_t \leq \max_{(x,y) \in \Xcal\times\Ycal} \frac{\pi_{t+1}^\star(y \mid x)\pi_{t+1}(y \mid x)}{\mu_{1,t}(y\mid x)\mu_{2,t}(y\mid x)}$ in the worst case. However, as $\pi_{t+1}$ is likely to be restricted within a certain region, for instance, because fine-tuning will not significantly alter the behavior of the language model, we anticipate that such a coefficient will not depend on the per-$(x,y)$ worst case.
On the other hand, as a direct observation, we notice that the ideal selection of $\mu_{1,t}$ and $\mu_{2,t}$ should be close to the target of soft policy iteration $\pi_{t+1}^\star$ (assuming $\pi_{t+1}^\star$ and $\pi_{t+1}$ are close). 
Interestingly, this theoretical observation coincides with recent empirical results. Here, \citet{liu2024statistical} suggests that using statistical rejection sampling to sample from the soft policy iteration target (which is almost equivalent to sampling $y_1$ and $y_2$ from $\pi_{t+1}^\star$) could benefit preference tuning.
However, in our case, if we use similar statistical rejection sampling techniques on $\pi_t$ to sample $\pi_{t+1}^\star$ (and $\pi_{t+1}$), the cost of rejection sampling is likely to be comparable to the concentrability coefficient $\Cfrak_t$ when choosing $\mu_{1,t}$ and $\mu_{2,t}$ to be $\pi_t$ \citep[see, e.g.,][]{wwen2013mcbook}. This suggests that both $\pi_t$ and $\pi_{t+1}^\star$ (via rejection sampling) as the choices of $\mu_{1,t}$ and $\mu_{2,t}$ will be comparable options in terms of sample efficiency.
On the other hand, as we will demonstrate in the next section, since $r_t$ is defined based on $\pi_t$ (as shown in \cref{step:def_rt} of \cref{alg:dno}), choosing $\mu_{1,t}$ and $\mu_{2,t}$ to be $\pi_t$ can easily adapt to such a reward of $r_t$.

Another interesting observation is that despite \cref{eq:reward_regression} sharing a similar form with Bradley-Terry style reward modeling with using MLE, the target distributions used to measure distribution shift appear to be quite different. This disparity is due to the different objectives: fitting soft policy iteration versus reward estimation. 
For the Bradley-Terry style reward modeling using MLE, the desired distribution of $y_1$ and $y_2$ should be two distinct distributions \citep[see, e.g.,][]{zhan2024provable,xiong2023iterative}. However, in our case where the learning goal is to fit the soft policy iteration, we may prefer $y_1$ and $y_2$ from two (near) on-policy distributions as discussed above, as long as we expect the learned $\pi_{t+1}$ will be accurate enough. To the best of our knowledge, this is the first theoretical result that illustrates the importance of on-policy sampling beyond policy optimization style algorithms for RLHF.

\section{\mbox{Practical Algorithm -- Iterative Contrastive Self-Improvement}}
\label{sec:prac_algo}

In this section, we shift our focus to the algorithmic design of the practically scalable version of \DNO, following the principles discussed in the last section.
A primary challenge encountered in the implementation of the conceptual algorithm \DNO (\cref{alg:dno}) stems from the necessity to compute the expectation with respect to the preference function $\mathcal{P}$ under the current policy $\pi_t$.
Perhaps surprisingly, as we will show, all we need is a properly implemented iterative \DPO-like contrastive learning algorithm.

\begin{algorithm}[th]
\caption{\NashDPO: Practical Implementation of \DNO via Iterative Contrastive Self-Improvement}
\label{alg:NashDPO}
{\bfseries input:} General preference function $\Pcal$, learning rate $\etatd$, iterations $T$, reference policy $\piref$, prompt distribution $\rho$.
\begin{algorithmic}[1]
\State Initialize $\pi_1 \leftarrow \piref$.
\For{iteration $t = 1,2,\dotsc,T$}
    \State {\bf Construct } $\Dcal_t = \{(x,y^\gold)\}$ where $x \sim \rho$ and $y \sim \pig(\cdot\mid x)$.
    \State {\bf Sample \emph{batched on-policy} responses:}\label{step:batched_op} Sample $K$ outputs per per prompt using the current $\pi_t$: $\{y_t^1,y_t^2,\dotsc,y_t^K\} \sim \pi_t(\cdot \mid x)$, $\forall x \in \Dcal_t$.
    \State {\bf Rank responses:}\label{step:dno_rank} For each $x \in \Dcal_t$, rank the corresponding $\{y_t^1,y_t^2,\dotsc,y_t^K,y^\gold\}$ using the pair-wise win-rate by sampling from the general preference function $\Pcal$.
    \State {\bf Filter preference pairs:}\label{step:filtering} Construct $\Dcal_{t+1} = \{(x,y_t^+,y_t^-)\}$, for all $x \in \Dcal_{t+1}$, and $(y_t^+,y_t^-)$ are large-margin pairs (based on the win-rate rank) within the responses for $x$ from the previous step.
    \State {\bf Contrastive learning:} Obtain $\pi_{t+1}$ by,
    \begin{gather}
    \label{eq:nashdpo}
    \pi_{t+1} \leftarrow \argmax_{\pi \in \Pi} \E_{(x,y_t^+,y_t^-) \sim \Dcal_{t+1}} \log \left[ \sigma\left( \etatd \log\frac{\pi(y_t^+ \mid x)}{\pi_t(y_t^+ \mid x)} - \etatd \log\frac{\pi(y_t^- \mid x)}{\pi_t(y_t^- \mid x)} \right)\right].
    \end{gather}
\EndFor
\State \Return best of $\pi_{1:(T+1)}$ on the validation data.
\end{algorithmic}
\end{algorithm}

We present our the practical implementation of \DNO in \cref{alg:NashDPO} (\NashDPO), which is a batched on-policy algorithm that conducts self-improvement iteratively via contrastive learning.
One key consideration in our algorithmic design is that we only need to implicitly use the reward function $r_t$. This comes from the specifically designed on-policy sampling, data filtering, and pair construction. 
While these specific design choices make \NashDPO seem similar to simply performing \DPO iteratively, there are significant reasons for these design decisions, as we will discuss below.

\paragraph{Batched on-policy sampling}
\label{paragraph:on-policy-sampling}
The use of batched on-policy sampling in \cref{step:batched_op} of \cref{alg:NashDPO} is crucial to avoid explicit use of $r_t$ (defined as $\E_{y' \sim \pi_t(\cdot \mid x)}\left[\Pcal(y \succ y' \mid x)\right]$ in \cref{step:def_rt} of \cref{alg:dno}). This means we essentially choose $\mu_1$ and $\mu_2$ in \DNO to be $\pi_t$ in \NashDPO, but we are free to let them vary slightly as a mixture of other policies, e.g., from a stronger teacher. 
Specifically, it is unrealistic to assume in practice that we can access the exact value of $\mathcal{P}(y \succ y' \mid x)$ given an $(x,y,y')$ tuple. 
Based on the definition of $r_t$ and the fact of $\{y_t^1,y_t^2,\dotsc,y_t^K\}$ are sampled from $\pi_t$, \NashDPO essentially uses the following sampled based approach to estimate $r_t$: $r_t(x,y) \approx \frac{1}{K}\sum_{y' \in \{y_t^1,y_t^2,\dotsc,y_t^K,y^\gold\} \setminus y} \1_\Pcal(\text{Is $y$ better than $y'$ on $x$}?)$, for any $x$ and $y \in \{y_t^1,y_t^2,\dotsc,y_t^K,y^\gold\}$, where $\1_\Pcal$ denotes one sample from $\Pcal$ and output $\{0,1\}$. This is implemented in \cref{step:dno_rank} of \cref{alg:NashDPO}, and its precise implementation on this is discussed in the \cref{sec:experiments}.
On the other hand, as we discussed in the last section, the batched on-policy sampling from $\pi_t$ is an appropriate option due to the consideration of sample efficiency when we use \cref{eq:nashdpo} to approximate the soft policy iteration (see \cref{thm:dno_guarantee_main} and its discussion). 

\paragraph{Preference pair construction}
\label{paragraph:preference-pair-construction}
Another key design choice in \cref{alg:NashDPO} is that \cref{eq:nashdpo} of \cref{alg:NashDPO} only uses the purely contrastive loss, whereas \cref{eq:dnoloss} of \cref{alg:dno} also contains the regression target $\sigma\left( r_t(x,y) - r_t(x,y') \right)$ (for a given $(x,y,y')$ tuple), which is not necessarily $\{0,1\}$.
As we discussed above, it is unrealistic to expect access to the exact value of $\mathcal{P}(y \succ y' \mid x)$, so it is also unlikely to get an accurate value of the regression target $\sigma( r_t(x,y) - r_t(x,y') )$. Thus, we add an additional data filtering step to address this issue as in \cref{step:filtering} of \cref{alg:NashDPO}.
Ideally, we want the selected $(x,y^+,y^-)$ tuple to satisfy $\sigma(r_t(x,y_t^+) - r_t(x,y_t^-)) \approx 1$, so that \cref{eq:dnoloss} can be approximated by \cref{eq:nashdpo}. However, one can notice that it requires $r_t(x,y_t^+) - r_t(x,y_t^-) \to \infty$, but we know $r_t(x,y) \in [0,1]$, $\forall (x,y) \in \Xcal \times \Ycal$.

From the derivation of \DNO in \cref{sec:dno}, it is clear that scaling up $r_t$ and $\eta$ with the same absolute constant $c$ does not affect the soft policy iteration target of \cref{eq:def_pistar}, but it will slightly change the \DNO objective (\cref{eq:dnoloss} in \cref{alg:dno}) by $r_t \to c \cdot r_t$ and $\eta \to c \cdot \eta \eqqcolon \etatd$.
This scaling strategy helps us sidestep the problem of bounded $r_t$, and in this sense, we may expect the proper $\etatd$ in \NashDPO to be relatively larger (than, e.g., $\eta$ in \cref{alg:dno}).
However, an enlarged $\etatd$ in \cref{eq:dnoloss} will worsen the sample complexity suggested in \cref{thm:dno_guarantee_main} (for details, refer to its proof in \cref{app:dno_proof}, especially for the derivation of \cref{eq:mean_value_thm}). So, to avoid the proper $\etatd$ being too large, \textbf{we only use pairs with large margin} as in \cref{step:filtering} of \cref{alg:NashDPO} to make sure $r_t(x,y_t^+) - r_t(x,y_t^-)$ is not too small. This decision is also supported empirically in techniques like RLCD~\citep{yang2023rlcd} and Axiomatic Preference Models~\citep{rosset2023axiomatic} which highlight the importance of having large margin or clear directional differences between positive and negative LLM responses when training preference models.

\paragraph{Relationship between \NashDPO and \DPO}
The reader may discern that \NashDPO~(\cref{alg:NashDPO})---the practical implementation of \DNO---can be described as an iterative  version of the \DPO algorithm.
Such similarity is by design, intended to harness the simplicity and effectiveness of DPO~\citep{rafailov2023direct} and build on  empirical advancements from recent work that applies DPO iteratively \citep[e.g.,][]{yuan2024selfrewarding,snorkelai2024}. Our experiments point to the importance of several design choices which help accommodate the general preferences, such as rankings derived from pair-wise win rates. More interestingly, our findings point to a surprising connection---that \textbf{\emph{``a meticulously designed iterative \DPO algorithm'' could approach the Nash equilibrium of any given general preferences.}}

Our general algorithmic framework---\DNO~(\cref{alg:dno})---is broader and fundamentally different from iterative \DPO. For example, the \DNO framework could also be directly extended to the regularized preference case (as discussed in \cref{app:nash_md_ver}) or equipped with other advanced sample techniques \citep[e.g.,][RSO]{liu2024statistical} as suggested by \cref{thm:dno_guarantee_main} for sample efficiency.
On the other hand, although the soft policy iteration (or the KL-regularized reward optimization) is used in both \DNO and \DPO, they arise from fundamentally different reasons.
For \DNO, KL-regularization originates from online learning, no-regret learning through mirror descent \citep{nemirovskij1983problem} or follow-the-regularized-leader (FTRL) \citep{kalai2005efficient,cesa2006prediction,shalev2012online,hazan2016introduction}.
For \DPO and \PPO, the KL-regularization is an approximation for the total variation penalty to ensure monotonic improvement of the policy \citep{kakade2002approximately,schulman2015trust}. Later, this approach was simplified by \citet[PPO]{schulman2017proximal}, and recently used for post-training LLMs \citep{ouyang2022training}.

\section{Experiments}
\label{sec:experiments}

\begin{figure}[t]
    \centering
    \includegraphics[width=0.75\textwidth]{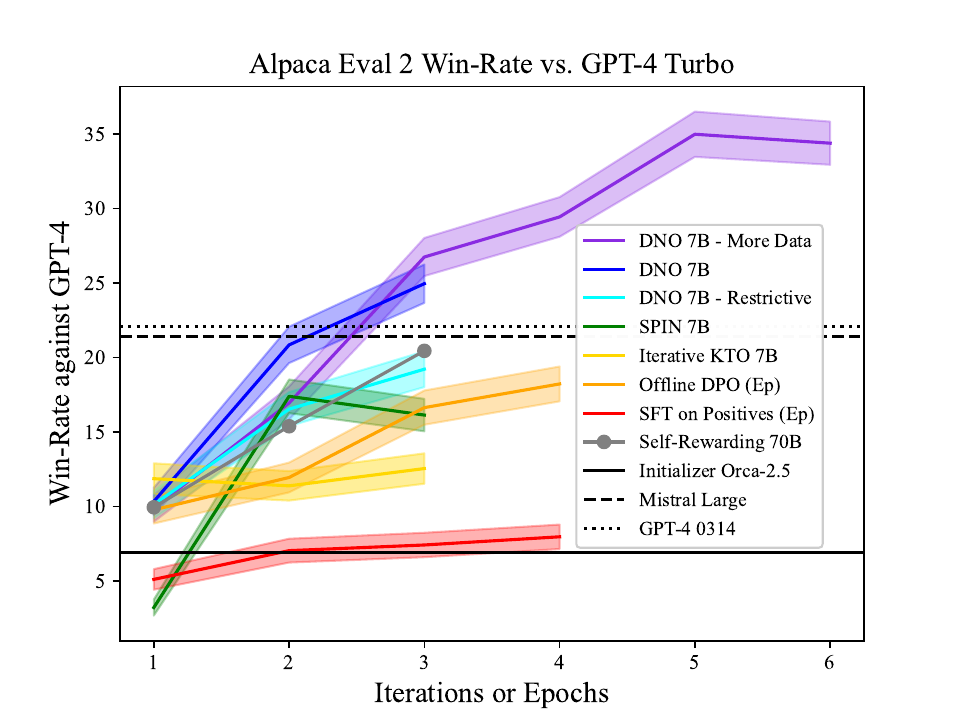}
    \caption{Comparison of various post-training techniques showing that \DNOfull (\DNO) is the most effective. All methods with colorful error bands are 1) implemented by ourselves, 2) initialized with a 7B parameter \orca LLM, and 3) are ``batched on-policy'' (except SFT and Offline \DPO which are epochs), all else being equal. }
    \label{fig:alpaca-2-plots-dno}
\end{figure}

\cref{alg:NashDPO} is chosen for its efficiency and simplicity from an implementation standpoint (in this section, we will use \DNO to denote \cref{alg:NashDPO} or \NashDPO for simplicity). Once the input dataset $\{x_i \in \Xcal\}$ is chosen, each iteration of \DNO proceeds in three phrases: sampling outputs from the current policy, annotating outputs for preference pair generation, and then training the next policy with the new training pairs. Iteration 0 is defined to start by sampling from the initial SFT model to produce training data for iteration 1. 

\subsection{Experimental Setup}
\textbf{Data}: We mainly use Ultrafeedback~\citep{cui2023ultrafeedback}, which consists of 60k prompts, several models' outputs to those prompts, and preference annotations from \gptturbo. This dataset thus provides a source of offline preferences. For our iterative experiments, we split this dataset into three non-overlapping partitions of the inputs to be used for separate iterations of batched on-policy learning. For each input, we also collect the \gptturbo output if it was not already present in the original dataset to be reserved for $y^{\text{gold}}$. 

Every experiment except one in this study solely uses UltraFeedback. The exception is one ``scaled up'' experiment with about 10x more data sourced from a mixture of datasets aggregated including  Anthropic HH ~\citep{bai2022training}, UltraChat ~\citep{ding2023enhancing}, MetaMathQA~\citep{yu2023metamath}, EvolInstruct~\citep{xu2023wizardlm},  UltraFeedback~\citep{cui2023ultrafeedback} and Orca-2~\citep{mitra2023orca}. Note that we only use the input prompts for these datasets and collect a \gptturbo responses for all ~600k of these input prompts.  

\textbf{Sampling from the Policy}: At the end of training, we sample 5 outputs from the resulting student policy using top p sampling with $p=0.95$ and temperature 0.7. Several works have shown the benefit of sampling and comparing multiple diverse outputs from the policy~\citep{yuan2023scaling, mitra2024orcamath, liu2024statistical, dong2023raft, wang2022self}. We implement a simple defect detection system which flags any sample that has a high amount of repeated n-grams as automatic negative. 

\textbf{Preference Annotation}: We use \gptturbo ``as a judge'' to label preferences among the 5 policy samples and 1 gold sample (which is also \gptturbo) as shown in \cref{fig:annotation_prompt}. This prompt contains a few minor modifications from the that used in~\citep{yuan2024selfrewarding}. It implements an additive scoring framework on a 6-point scale where a score of 6 represents the highest quality answer according to certain dimensions like ``correctness'', ``expert knowledge'', ``conciseness'' etc. By following this rubric, GPT-4 acting as an annotator represents a best-effort general preference model because it compares multiple candidate responses side-by-side in the context window, and stratifies them along meaningful dimensions of quality.

\begin{table}[t]
\centering
\begin{tabular}{ll|ccc|ccc}
 &  & \multicolumn{3}{c|}{Alpaca Eval 2} & \multicolumn{3}{c}{MT Bench} \\
Technique & \begin{tabular}{@{}c@{}}Epoch \\ or Iter\end{tabular} & \begin{tabular}{@{}c@{}}Len-control. \\ Win Rate\end{tabular} & \begin{tabular}{@{}c@{}}Win Rate \\ vs. GPT-4\end{tabular} & \begin{tabular}{@{}c@{}}Avg. len \\ (chars)\end{tabular} & \begin{tabular}{@{}c@{}}1st \\ Turn\end{tabular} & \begin{tabular}{@{}c@{}}2nd \\ Turn\end{tabular} & Avg \\ \hline
Orca-2.5 SFT & Epoch 1 & 10.76 & 6.99 & 1174 & 7.72 & 6.02 & 6.88 \\
Orca-2.5 SFT on Positives & Epoch 4 & 11.62 & 7.96 & 1420 & 7.62 & 6.23 & 6.92 \\
Offline \DPO (ours) & Epoch 4 & 19.49 & 18.22 & 1884 & 7.69 & 7.08 & 7.38 \\ \hline
Self-Rewarding 70B  & Iter 3 & -  & 20.44 & 2552 & - & - & 7.25 \\
\SPIN (ours)       & Iter 3 & 16.18 & 16.13 & 1922 & 7.58 & 7.53 & 7.55 \\
\DNOrs & Iter 3 & 21.61 & 19.21 & 1795 & 7.59 & 7.35 & 7.46 \\
\DNOlh & Epoch 1 & 18.58 & 18.28 & 1907 & 8.09 & 7.32 & \textbf{7.70} \\ \hline
\DNO             & Iter 3 & \textbf{22.59} & \textbf{24.97} & 2228 & 7.62 & 7.35 & 7.48 \\

\end{tabular}
\caption{\alpaca and MT-Bench results in our controlled setting after training on UltraFeedback.}
\label{tab:ablation-run-alpaca-win-rates}
\end{table}

\textbf{Training Pair Construction}: 
Adhering to~\cref{step:filtering} in \cref{alg:NashDPO} implies that not all pairs are suitable for training. Firstly, we must enforce the positives to be high quality in an absolute sense, and secondly, the negatives are directionally worse by a large margin. On the 6 point annotation scale, only samples that score a 5 or 6 are allowed to be positives. From the positives that meet this criteria, if any, we then construct all pairs such that the negative is at least 2 points lower. If the positive happens to be from the student, we relax this constraint to 1 point margin since the \gptturbo teacher outputs rarely receive a score less than 5 (as shown by the average teacher score in~\cref{tab:count-of-training-pairs}). 

Additionally, we are motivated to preserve the preference behavior from previous iterations so that new policies do not inadvertently regress to past bad behavior. To enforce this, we incorporate an exponentially decaying proportion of prior iterations' training pairs into the current iteration, i.e. we sample at most 30\% of training pairs from iteration $t-1$, 15\% from $t-2$, and so on. We do not re-inference outputs for those inputs from the most recent policy. Recall that previous iterations' inputs are non-overlapping with the splits for other iterations. 

\textbf{Training}: To prevent overfitting, we train our batched on-policy  methods for at most one epoch on newly constructed pairs. Our effective batch size is fixed to 64 for all experiments. Our learning rate, beta, and alpha are found with brief hyperparameter searches. For most experiments, the learning rate is 5E-5, beta is either 0.1 or 0.05, and alpha is 0.005. We found that at higher iterations, the learning rate needs to be lowered. In SFT (supervised fine-tuning) experiments, our learning rate is 5E-6 and we mask out loss for the inputs. We use the open-source \href{https://huggingface.co/docs/trl/main/en/dpo_trainer}{TRL} library's implementation to run our experiments. 

\textbf{Evaluation}: Our primary goal is to train a policy that is comparable to the most powerful state-of-the-art langauge models. Hence, \alpaca \citep{dubois2023alpacafarm} is an appropriate benchmark because it computes win-rate against \gptturbo in a head-to-head fashion on a dataset of 805 input prompts that is shown to correlate with human preferences (0.93 spearman correlation with Chatbot Arena). While it is known that auto-eval methods also correlate with spurious features such as length, a new version of \alpaca corrects for this with a length-controlled win-rate that has an even higher spearman correlation (0.98) with Chatbot Arena~\footnote{\url{https://github.com/tatsu-lab/alpaca_eval}}. 

We also evaluate on MT-Bench~\citep{zheng2023judging} which allows the llm-as-a-judge to first explain its reasoning before providing a scalar score on 1-10 for the candidate response to a bank of 80 questions. One crucial difference between \alpaca and MT Bench is that the former asks \gptturbo to predict which of two side-by-side responses humans would prefer, weighted by the \textit{logits} to represent its uncertainty, whereas MT-Bench asks the model to first \textit{generate} a justification and then output a score on 1-10, but it neither defines the ratings (e.g. how a 7 is different than a 5) nor accounts for uncertainty in the logits of the score.

We also evaluate on the OpenLLM leaderboard~\citep{open-llm-leaderboard}, which measures reasoning ability on downstream NLP tasks like coding and question answering by evaluating the accuracy of the multiple choice answer option with the highest logit. Since our training data is primarily instruction-following and not trained to output just the sole answer option, this benchmark is not the primary target of this study; nonetheless, \DNO on instruction tuning tasks ought to show no regression on reasoning tasks. 

\begin{table}[t]
\centering
\setlength{\tabcolsep}{4pt} %
\begin{tabular}{c|cc|cccc|ccc}
 & \multicolumn{1}{l}{} & \multicolumn{1}{l|}{} & \multicolumn{4}{c|}{Annotations of Training Data} & \multicolumn{3}{c}{New Training Pairs} \\
 & \multicolumn{1}{l}{inputs} & \multicolumn{1}{l|}{\begin{tabular}{@{}c@{}c@{}}student \\ length \\ (words)\end{tabular}} & \multicolumn{1}{l}{\begin{tabular}{@{}c@{}c@{}}best-of-n \\ student \\ win-rate\end{tabular}} & \multicolumn{1}{l}{\begin{tabular}{@{}c@{}c@{}}Avg. \# \\ student \\ wins\end{tabular}} & \multicolumn{1}{l}{\begin{tabular}{@{}c@{}c@{}}Avg. \\ student \\ score\end{tabular}} & \multicolumn{1}{l|}{\begin{tabular}{@{}c@{}c@{}}Avg. \\ teacher \\ score\end{tabular}} & \multicolumn{1}{l}{$T \succ S$} & \multicolumn{1}{l}{$S \succ T$} & \multicolumn{1}{l}{$S \succ S$} \\ \hline
\DNOrs Iter 0 & 19.6k & 162 +/- 190 & 15.9\% & 0.486 & 3.46 & 4.99 & 42.4k & {\color{red} 0} & {\color{red} 0} \\
\DNOrs Iter 1 & 19.9k & 359 +/- 350 & 34.2\% & 1.11 & 4.86 & 4.77 & 17.5k & {\color{red} 0} & {\color{red} 0} \\
\DNOrs Iter 2 & 19.8k & 256 +/- 207 & 35.0\% & 1.31 & 5.21 & 4.87 & {\color{red} 9.9k} & {\color{red} 0} & {\color{red} 0} \\ \hline
\DNO Iter 0 & 19.6k & 162 +/- 190 & 15.9\% & 0.486 & 3.46 & 4.99 & 30.7k & 4.2k & 25.9k \\
\DNO Iter 1 & 19.9k & 671 +/- 546 & 34.6\% & 1.22 & 4.61 & 4.62 & 20.3k & 19.4k & 62.9k \\
\DNO Iter 2 & 19.8k & 361 +/- 251 & 43.6\% & 1.90 & 5.25 & 4.59 & 7.1k & 32.4k & 10.9k
\end{tabular}
\caption{The dynamics of how sampled outputs from a previous iteration's policy compare to their teacher, and how many new training pairs they give rise to in the next iteration. The crucial point is that \DNO constructs new pairs where the student is compared to the teacher, whereas \DNOrs, \SPIN, and IPO-MD do not. }
\label{tab:count-of-training-pairs}
\end{table}

\subsection{Results and Analysis}
\label{sec:results}

We run several head-to-head experiments that control for hyperparameters and input data. We often refer to the policy being trained as the ``student'' and GPT-4 as a ``teacher''; GPT-4 is also used as an annotator when prompted. 

\textbf{SFT Baselines}
The first baseline is \orca itself, which is a \texttt{mistralai/Mistral-7B-v0.1} raw pretrained model fine-tuned on a new collection of Orca-2 data~\citep{mitra2023orca}. This model was finetuned for three epochs and achieves scores shown in the top of \cref{tab:big-run-alpaca-win-rates}. All other experiments in this study are initialized with Epoch 1 of \orca. This is the solid horizontal line in \cref{fig:alpaca-2-plots-dno}. 

The second baseline is continue-SFT of \orca training towards the positives in UltraFeedback (and masking out loss over the input prompts). If the original positive in that dataset was not from \gptturbo, we replace it with one that is. This is the red line in \cref{fig:alpaca-2-plots-dno}. It is clear that even offline contrastive training methods are more beneficial than additional SFT, showing that the \textit{difference} between the positive and negative output provides more valuable training signal than the positive in isolation. 

\textbf{Large Margin Filtering of Training Pairs}: We ran a simple experiment of Offline \DPO for one epoch on UltraFeedback data. In the control, we trained on all 63k preference pairs in the original dataset, whereas in the treatment we filtered the 42k pairs that met a large margin requirement enforcing that the positive's scores exceeded that of the negative by at least 1.0 (out of 10) according to their \gptturbo annotator. All else was equal. Even though the treatment was trained for fewer steps on less data, it achieved an \alpaca win rate of 11.60 vs 9.60 for the control, showing that fewer higher quality preference pairs is better than a higher quantity of noisy pairs (not shown in the tables). 

\textbf{On-Policy is Better than Off-Policy} One of the critical questions in this study whether to sample ``on-policy'' outputs from the current student to use in training pairs, or whether ``off-policy'' outputs collected from other models different than the student will suffice. We ran 4 epochs of Offline \DPO on UltraFeedback (filtered for large margin), and as shown in \cref{tab:ablation-run-alpaca-win-rates}, on-policy methods especially \DNO surpass the off-policy \DPO, even when trained for 4 epochs while the on-policy models were granted only three iterations. Recall that each iteration of batched on-policy training sees only a third of the UltraFeedback input data, whereas an epoch of Offline \DPO sees the entire dataset. 

\begin{table}[t]
\begin{tabular}{c|cccccc|l}
\multicolumn{1}{r|}{} & \begin{tabular}{@{}c@{}c@{}}ARC-C \\ (25-shot)\end{tabular} & \begin{tabular}{@{}c@{}c@{}}GSM8K \\ (5-shot)\end{tabular} & \begin{tabular}{@{}c@{}c@{}}HellaSwag \\ (10-shot)\end{tabular} & \begin{tabular}{@{}c@{}c@{}}MMLU \\ (5-shot)\end{tabular} & \begin{tabular}{@{}c@{}c@{}}TruthfulQA \\ (0-shot)\end{tabular} & \begin{tabular}{@{}c@{}c@{}}WinoGrande \\ (5-shot)\end{tabular} & \multicolumn{1}{c}{Avg} \\ \hline
Orca-2.5 Epoch 1 & 0.609 & 0.635 & 0.818 & 0.614 & 0.489 & 0.738 & 0.652 \\
Orca-2.5 Epoch 3 & 0.624 & 0.641 & 0.826 & 0.624 & 0.506 & 0.746 & 0.661 \\ 
\SPIN (ours) Iter 3 & 0.668 & 0.448 & 0.862 & 0.623 & 0.601 & 0.759 & 0.660 \\ \hline
\DNO Iter 1 & 0.657 & 0.572 & 0.834 & 0.623 & 0.568 & 0.755 & 0.668 \\
\DNO Iter 2 & 0.663 & 0.562 & 0.845 & 0.624 & 0.580 & 0.753 & 0.671 \\
\DNO Iter 3 & 0.672 & 0.542 & 0.852 & 0.622 & 0.606 & 0.753 & 0.675
\end{tabular}
\caption{Results on Open-LLM Leaderboard reasoning tasks, which we do not expect to decrease.}
\label{tab:open-llm-leaderboard}
\end{table}

\textbf{Higher Quality Annotators}
\label{sec:better-annotator}
In our study, we use \gptturbo to provide the annotations for preference pairs. However, the Self-Rewarding Language Model uses the \llama \citep{touvron2023llama} model trained to also give feedback as the annotator, which in their study starts off with a 65\% agreement rate with human-labeled preferences improving to 80\% in the last iteration~\citep{yuan2024selfrewarding}. While it was not reported how well \gptturbo's annotations agree with their held-out human labels, we believe that having a higher-quality annotator to start with will lead to higher quality policies. Since both our studies use UltraFeedback data, and our annotation prompt is based on their annotation prompt, we believe there is a valid comparison. 

We observe \DNO initialized with a 7B base model outperforms the 70B parameter Self-Rewarding model over the same number of training iterations (24.97 win-rate vs 20.44 on \alpaca, and 7.46 MT-Bench vs 7.25), at least in part due to the higher quality preference annotations. See the dark blue band versus the gray line in \cref{fig:alpaca-2-plots-dno} and the corresponding row in \cref{tab:ablation-run-alpaca-win-rates}. However, unlike Self-Rewarding LM, we saw a slight gain rather than a drop reasoning benchmarks like ARC-Challenge~\citep{allenai:arc} and HellaSwag~\citep{zellers2019hellaswag}. Granted, the evaluation of OpenLLM predicts the answer with the max logit corresponding to one of the multiple-choice options, which is not congruous with how these techniques are trained.

\textbf{Training Pair Construction}
One of the most critical implementation questions in this study is \textbf{how to construct training pairs that help the student policy exceed a strong teacher} like \gptturbo. One approach, Self-Play Finetuning (\SPIN), removes the preference annotation step and automatically assigns the teacher output to be the positive, and all student samples to be negative~\citep{chen2024selfplay}. We find in our re-implementation of \SPIN that this is detrimental, presumably because this automatic assignment could lead to noisy training pairs in cases where the student might actually be preferred. The resulting win-rate of \SPIN is only 16.13 after three epochs of iterative training compared to 24.97 for \DNO as shown in  \cref{tab:ablation-run-alpaca-win-rates}, all else being equal. Similar results hold in the OpenLLM results in~\cref{tab:open-llm-leaderboard}.

In a second experiment, which we denote \DNOrs, we annotate all preference pairs with \gptturbo as usual, but only admit training pairs where the teacher's output is the preferred one. The difference between \DNO and \DNOrs is illustrated in  \cref{tab:count-of-training-pairs} where 0 student-vs-teacher and student-vs-student pairs are created. The same is also true for \SPIN, but  \SPIN would admit a greater quantity of noisy teacher-vs-student examples even when they are dis-preferred: \cref{tab:count-of-training-pairs} shows that after Iteration 2 of \DNOrs, only 9.9k instances exist of the teacher being preferred over the student, whereas \SPIN would have automatically created about 100k (5 samples $\times$ 20k inputs). 

While \DNOrs is slightly better (19.21 win-rate) than \SPIN, it still does not give the student a chance to compare its behavior to a powerful teacher. Absence of this signal is a major oversight, since the last row of  \cref{tab:count-of-training-pairs} shows that by Iter 3, over 64\% of the \DNO training data (32k pairs) are cases where the student is in fact preferred over the teacher, a number which increases with iteration.  \textbf{We conclude it is imperative to ``\textit{allow the student to become the teacher}'' i.e. learn from comparisons where its own outputs are preferred over a more powerful teacher}. 

One curious phenomenon in~\cref{tab:count-of-training-pairs} is that while the teacher outputs are fixed ahead of time, the annotator gives slightly lower scores to the teacher as the student improves; we are not sure if this is an innocuous artifact of preference annotations, or symptomatic of a deeper problem. Also, the total quantity of new ``large margin'' training pairs (not counting those sampled from previous iterations) in \DNO tends to decrease as the policy improves across iterations, but we do not have enough data to quantify how this relates to a change in quality.

\textbf{Lookahead to Future Iterations}
As a curiosity, we experimented with whether a model could benefit from the knowledge of which training pairs it would generate if it could look into the future.  
We tested this by running three-iterations of \DNO, accumulating all the preference pairs across iterations, combining and shuffling them, and then re-starting training from the initial model. In essence, this turns the batch-online \DNO into an offline learning algorithm we denote as \DNOlh. We trained for one epoch on the three iterations' worth of preference data. It deteriorated more than we expected on \alpaca win-rate (24.97 to 18.18), however, even more surprisingly, the MT-Bench numbers improved significantly (7.48 to 7.70). 
While the reasons for the relatively low correlation between MT-Bench and \alpaca are not entirely clear, it is important to consider the disparity in the size of the datasets. Given that MT-Bench consists of merely 80 examples, whereas \alpaca contains 10x more, we conjecture that the statistical significance and reliability of the findings from \alpaca are regarded with greater confidence.

\textbf{DNO Scales with More Data}:
One of the reasons we split UltraFeedback into three non-overlapping partitions is to avoid overfitting. Another strategy to avoid overfitting is to collect more data, so we increased by a factor of 10 the instruction data based on publicly available datasets. We split a large mixture of datasets into six non-overlapping partitions of roughly 100k inputs each (and inference \gptturbo outputs for all inputs), and show that \DNOmore scales well in this expanded regime (see the purple line in \cref{fig:alpaca-2-plots-dno} and the last row of \cref{tab:big-run-alpaca-win-rates}. 

We make some notes on the behavior of this experiment: because each iteration builds on outputs of the previous iteration, if there are any anomalies or errors in critical components such as preference annotation, those errors will propagate and the only way to combat them is ``roll back'' to the iteration that introduced them. This can result in wasted time and cost, which are both already very high as shown in~\cref{additional-experimental-details}. We suspect that the ``depth'' of iterations matters more than the ``width'' or number of samples within each iteration, and furthermore, that having equal number of inputs per iteration may not be optimal, but we did not test this thoroughly. From an efficiency standpoint, although this algorithm is ``batched'', some optimizations can be made, such as starting to annotate sampled policy outputs are soon as they are ready instead of waiting for all inference jobs to finish. 

\textbf{``Exploding'' Lengths}
It is known that contrastive LLM training techniques, especially \DPO, lead to longer outputs from the model which is widely suspected to be a form of ``reward hacking''. Curiously, ~\cref{tab:count-of-training-pairs} shows that the largest jump comes after the first round of contrastive training (Iteration 1), where lengths explode by at least a factor of 2 over the initializing SFT model, before inching down again in the next iteration. We interpret this ``length spike'' as wasted computation optimizing towards a spurious signal; we wish we were better equipped to control this phenomenon. 

\begin{table}[t]
\centering
\begin{tabular}{cc|ccc|ccc}
 &  & \multicolumn{3}{c|}{Alpaca Eval 2} & \multicolumn{3}{c}{MT Bench} \\
Technique & \begin{tabular}{@{}c@{}}Epoch \\ or Iter\end{tabular} & \begin{tabular}{@{}c@{}}Len-control. \\ Win Rate\end{tabular} & \begin{tabular}{@{}c@{}}Win Rate \\ vs. GPT-4\end{tabular} & \begin{tabular}{@{}c@{}}Avg. len \\ (chars)\end{tabular} & \begin{tabular}{@{}c@{}}1st \\ Turn\end{tabular} & \begin{tabular}{@{}c@{}}2nd \\ Turn\end{tabular} & Avg \\ \hline
Orca-2.5 SFT & Epoch 1 & 10.76 & 6.99 & 1174 & 7.72 & 6.02 & 6.88 \\
Orca-2.5 SFT & Epoch 2 & 15.29 & 7.88 & 1060 & 7.56 & 6.38 & 6.98 \\
Orca-2.5 SFT & Epoch 3 & 15.90 & 8.17 & 1058 & 7.53 & 6.73 & 7.13 \\ \hline
\DNOmore & Iter 1 & 8.96 & 10.67 & 2795 & 7.00 & 6.06 & 6.53 \\
\DNOmore & Iter 2 & 14.61 & 16.94 & 2782 & 7.62 & 7.23 & 7.42 \\
\DNOmore & Iter 3 & 21.81 & 26.74 & 2539 & 7.74 & 6.66 & 7.21 \\
\DNOmore & Iter 4 & 22.93 & 29.08 & 3033 & 7.54 & 6.92 & 7.24 \\
\DNOmore & Iter 5 & 32.06 & 34.98 & 2856 & 7.10 & 6.39 & 6.75 \\
\DNOmore & Iter 6 & 33.05 & 34.38 & 2683 & 7.28 & 6.65 & 6.97
\end{tabular}
\caption{\DNOmore is trained on 10x more instruction data than \DNO. It is still initialized with Epoch 1 of \orca SFT, so the delta it provides in \alpaca win rate is 27.39 absolute (22.29 length-controlled)}
\label{tab:big-run-alpaca-win-rates}
\end{table}

\section{Related Work}

We divide the space of related work into whehter or not the techniques use SFT or contrastive losses, in offline or online update settings. 

\textbf{Online RLHF algorithms}:
RLHF innovated how to align language models with human preferences~\citep{christiano2017deep,stiennon2020learning}, but it is unstable to train and memory-intensive, requiring all three of the parameterized policy model, reward model, and advantage model to be on device for training.

\textbf{Reward-model Augmented SFT}:
Since the introduction of RLHF, several emergent techniques apply reward models in various ways, such as to filter training data or rank responses. Reward rAnked Finetuning (RAFT)~\citep{dong2023raft} and RRHF~\citep{yuan2023rrhf} offer the conceptually simplest solution for offline preference learning, which is to sample multiple outputs from a policy, rank them with a reward model, and then finetune on the best sampled output using SFT. This resembles the iterative behavior-cloning technique DAgger~\citep{ross2011reduction}. 

\textbf{Offline Contrastive Preference Learning}:
There exist several loss functions for contrastive preference learning, first introduced in the offline setting, namely Direct Preference Optimization~\citep[DPO]{rafailov2023direct} and Calibrated Sequence Likelihood Estimation a.k.a.~SLiC~\citep{zhao2023calibrating}.~\citet{azar2023general} make it clear that point-wise reward estimates are no substitute for pair-wise preferences, and that a policy can easily overfit to deterministic preferences without proper regularization. They derive a more general objective for RLHF, IPO, to directly optimize offline preference probabilities. 

Statistical Rejection Sampling Optimization (RSO) generates multiple samples from an initial model, ranks them to create training pairs, and optimizes them under a unified framework encompassing DPO and SLiC~\citep{liu2024statistical}. Inspired by the learning-to-rank literature, Listwise preference optimization (LIPO) extends pair-wise preference learning to list-wise~\citep{liu2024lipo}. Preference Ranking Optimization (PRO) also learns towards list-wise preferences~\citep{song2024preference}. The KTO algorithm takes a different approach from DPO and does not assume that a pair of good-vs-bad outputs for the same input exist, but rather a pool of good outputs and a pool of bad outputs for any inputs exist and optimizes an ``unpaired'' loss~\citep{ethayarajh2024kto}. 

\textbf{Iterative Reward-based Finetuning}:
Reinforced Self-Training (ReST) is one of the first methods to explore iterative self-improving training strategies framed as a two-stage ``Grow'' step that samples from the current policy, and a ``Improve'' step that uses a reward model to filter ever-higher quality samples that are then used to improve the policy with offline RL~\citep{gulcehre2023reinforced}. A follow-up work explores the use of AI feedback rather than reward ranking~\citep{singh2023beyond}.

\textbf{On-policy Contrastive Learning}:
Self-Rewarding Language Models~\citep{yuan2024selfrewarding} is in practice very similar to \DNO. They study the benefits of batched iteratively training on preferences derived from a recent policy's sampled outputs, but in their work, they use the policy itself as the annotator, which starts off being able to provide only weak preference signals. Self-Play Fine-Tuning~\citep{chen2024selfplay} a.k.a \SPIN and Adversarial Preference Optimization a.k.a APO~\citep{cheng2023adversarial} are both iterative LLM training techniques that are compatible with contrastive losses, but they make a very limiting assumption that the teacher is better than the student (without regard to any annotator feedback).

The Cringe Loss~\citep{adolphs2022cringe} is a token-level loss function that contrasts the correct next token with a hard-negative token from the vocabulary that has high logit weight but still incorrect. The Pairwise Cringe Loss~\citep{xu2023some} applies the cringe loss to an iterative self-improving style of training.

\textbf{On-Policy General Preference Optimization:}
\citet{wang2023rlhf} consider finding the von Neumann winner of general preferences via multi-agent RL from the theoretical perspective.
Nash-MD optimizes a policy towards the Nash equilibrium of a generalized preference model using policy gradients, showing that by sampling from a mixture of policies, one can converge to the Nash equilibrium in the last iteration~\citep{munos2023nash}. Self-play Preference Optimization (SPO) is another online two-player mini-max game that converges to a Nash equilibrium with no-regret guarantees~\citep{swamy2024minimaximalist}. However, these techniques are not as data efficient as contrastive losses and are difficult to implement faithfully without cumbersome two-timescale updates~\citep{munos2023nash}. 
A concurrent improvement, IPO-MD, mitigates these difficulties by using purely on-policy IPO updates and is empirically evaluated on an article summarization task \citep{calandriello2024human}.
\citet{guo2024direct} also propose to eliminate rewards in online AI-feedback (OAIF) by using another LLM to annotate which of two online-sampled outputs from the current policy is preferred. However, all the above studies only consider training pairs constructed between self-play ``student vs student'' samples, and between student and initial $\pi_{\text{ref}}$. That is, there is no concept of a more powerful ``teacher'' to compare against in their training pairs. We showed in~\cref{tab:count-of-training-pairs} that omitting these ``student vs teacher'' preferences may hinder performance.

\section{Conclusion}
\label{sec:conclusion}
In this paper we achieve dual goals of post-training LLMs against a more general class of preference models while providing a practical and scalable implementation with finite-sample analysis.
Our strong empirical results are based on the insight that optimizing general preference functions can be reduced to finding the Nash equilibrium of a two-player game with the payoff as the preference, and further solved by a single-play algorithm.
Most techniques to optimize for this objective use soft policy iteration, which is difficult to implement faithfully and may require unstable on-policy and two-timescale updates. Our contribution, \DNOfull, addresses these challenges by approximating soft policy iteration updates with a regression-based contrastive objective in a batched manner, which is a much more stable and forgiving learning objective, and we establish a concentration bound of $\widetilde{O}(\nicefrac{1}{N})$ on the squared total variation error between the learned policy and its target of the soft policy iteration update at any given iteration $t$. Theoretically, \DNO converges to the Nash equilibrium on-average, but in practice enjoys monotonic improvement across iterations. 
Training a 7B parameter LLM with \DNO achieves state-of-the-art performance on \alpaca, exceeding both Mistral Large and older versions of GPT-4. We illuminate many of the practical design choices that will aid future development of iterative self-improving algorithms.

\clearpage
\bibliographystyle{plainnat}
\bibliography{ref}

\clearpage

\appendix
\onecolumn

\begin{center}
{\LARGE Appendix}
\end{center}

\section{Extension to Regularized Preferences}
\label{app:nash_md_ver}

In this section, we discuss how to extend the \DNO framework to the case of regularized preferences (defined in \cref{eq:def_reg_pref}),
\begin{align*}
\Pcal^{\pi,\pi'}_{\tau}(y \succ y' \mid x) = \Pcal(y \succ y' \mid x) - \tau \log\frac{\pi(y \mid x)}{\piref(y \mid x)} + \tau \log\frac{\pi'(y \mid x)}{\piref(y \mid x)},
\end{align*}
which was first introduced and solved by \citet{munos2023nash} via \NashMD introduced earlier.

One can notice that the only difference between \SPO and \NashMD is that \SPO uses the last iteration policy $\pi_t$ for both constructing reward $r_t$ and performing a soft policy iteration update, whereas \NashMD uses the smoothed version $\pi_t^\tau$ (firstly defined in \cref{eq:def_smooth_pit}),
\begin{align}
\label{eq:def_pitsmooth_app}
\pi_t^\tau(y \mid x) \coloneqq \frac{\pi_t(y \mid x)^{1 - \nicefrac{\tau}{\eta}}\piref(y \mid x)^{\nicefrac{\tau}{\eta}}}{\sum_{y' \in \Ycal}\pi_t(y \mid x)^{1 - \nicefrac{\tau}{\eta}}\piref(y \mid x)^{\nicefrac{\tau}{\eta}}}, ~ \forall (x,y) \in \Xcal\times\Ycal,
\end{align}
for both. This allows \NashMD to obtain a late-iteration guarantee.

On the other hand, due to the symmetry of regularized preferences, if we consider on-average convergence case, it is likely that \SPO can be adapted with a simpler way as follows:
for each $t = 1,2,\dotsc,T$,
\begin{align*}
    \text{(i)} \quad &~ r_t(x,y) \leftarrow \E_{y' \sim \pi_t(\cdot \mid x)}\left[\Pcal(y \succ y' \mid x)\right],~~\forall (x,y) \in \Xcal \times \Ycal
    \\
    \text{(ii)} \quad &~ \pi_{t+1}(\cdot \mid x) \leftarrow \frac{1}{Z_t(x)} \pi_t^\tau(\cdot \mid x) \exp\left(\frac{r_t(x,\cdot)}{\eta}\right),~~\forall x \in \Xcal,
\end{align*}
where $Z_t(x) \coloneqq \sum_{y \in \Ycal} \pi_t^\tau(y \mid x) \exp\left(\frac{r_t(x,y)}{\eta}\right)$ is the partition function for iteration $t$. 
Here, the smoothed policy $\pi_t^\tau$ is only used in the soft policy iteration step, and this coincides with the OMD algorithm from \citet{munos2023nash}.

\begin{algorithm}[ht]
\caption{\DNO (Regularized Preferences Version)}
\label{alg:dno_reg}
{\bfseries input:} General preference function $\Pcal$, learning rate $\eta$, coefficient of KL-regularization $\tau$, number of iterations $T$, prompt distribution $\rho$.
\begin{algorithmic}[1]
\State Initialize $\pi_1 \leftarrow \unif(\Acal)$.
\For{iteration $t = 1,2,\dotsc,T$}
    \State Compute $r_t(x,y)$ by,
    \Statex\quad\qquad{\color{txblue}Option I:}\algocmt{for on-average convergence}
    \Statex\quad\qquad\qquad$r_t(x,y) \leftarrow \E_{y' \sim \pi_t(\cdot \mid x)}\left[\Pcal(y \succ y' \mid x)\right], ~\forall (x,y)\in\Xcal\times\Ycal$.
    \Statex\quad\qquad{\color{txblue}Option II:}\algocmt{for last-iteration convergence}
    \Statex\quad\qquad\qquad$r_t(x,y) \leftarrow \E_{y' \sim \pi_t^\tau(\cdot \mid x)}\left[\Pcal(y \succ y' \mid x)\right], ~\forall (x,y)\in\Xcal\times\Ycal$,
    where $\pi_t^\tau$ is defined in \cref{eq:def_pitsmooth_app}.
    \State\label{step:dno_learnpi_reg} Obtain $\pi_{t+1}$ by,
    \begin{equation*}
    \begin{gathered}
    \pi_{t+1} \leftarrow \argmax_{\pi \in \Pi} \E_{(x,y_1,y_2) \sim \Dcal_t}\bigg\{ \sigma\left( r_t(x,y_1) - r_t(x,y_2) \right) \log \left[ \sigma\left( \eta \log\frac{\pi(y_1 \mid x)}{\widetilde\pi_t^\tau(y_1 \mid x)} - \eta \log\frac{\pi(y_2 \mid x)}{\widetilde\pi_t^\tau(y_2 \mid x)} \right)\right]
    \\
    \hspace{120pt} + \sigma\left( r_t(x,y_2) - r_t(x,y_1) \right) \log \left[ \sigma\left( \eta \log\frac{\pi(y_2 \mid x)}{\widetilde\pi_t^\tau(y_2 \mid x)} - \eta \log\frac{\pi(y_1 \mid x)}{\widetilde\pi_t^\tau(y_1 \mid x)} \right) \right] \bigg\},
    \end{gathered}
    \end{equation*}
    where $\Dcal_t$ is generated by $x \sim \rho, y_1 \sim \mu_{1,t}(\cdot \mid x), y_2 \sim \mu_{2,t}(\cdot \mid x)$ with some policies $\mu_{1,t}$ and $\mu_{2,t}$, and $\widetilde\pi_t^\tau(y\mid x) \coloneqq \pi_t(y \mid x)^{1 - \nicefrac{\tau}{\eta}}\piref(y \mid x)^{\nicefrac{\tau}{\eta}}, \forall (x,y)\in\Xcal\times\Ycal$ (the unnormalized version of $\pi_t^\tau(y\mid x)$ defined in \cref{eq:def_pitsmooth_app}).
\EndFor
\State \Return $\pibar = \unif(\pi_{1:T})$.
\end{algorithmic}
\end{algorithm}

\begin{algorithm}[ht]
\caption{\NashDPO (Regularized Preferences Version)}
\label{alg:NashDPO_reg}
{\bfseries input:} General preference function $\Pcal$, learning rate $\etatd$, coefficient of KL-regularization $\tau$, number of iterations $T$, reference policy $\piref$, seed dataset $\Dcal_0 = \{(x,y^\gold)\}$ where $x \sim \rho$ and $y \sim \pig(\cdot\mid x)$, reference model $\piref$.
\begin{algorithmic}[1]
\State Initialize $\pi_1 \leftarrow \piref$.
\For{iteration $t = 1,2,\dotsc,T$}
    \State {\bf Sample \emph{batched on-policy} responses:}\label{step:batched_op_reg} 
    \Statex\quad\qquad {\color{txblue}Option I:} \algocmt{for on-average convergence}
    \Statex\quad\qquad\qquad Sample $K$ outputs per per prompt using the current $\pi_t$: $\{y_t^1,y_t^2,\dotsc,y_t^K\} \sim \pi_t(\cdot \mid x)$, $\forall x \in \Dcal_0$.
    \Statex\quad\qquad {\color{txblue}Option II:} \algocmt{for last-iteration convergence}
    \Statex\quad\qquad\qquad Sample $K$ outputs per per prompt using the smoothed current policy $\pi_t^\tau$: $\{y_t^1,y_t^2,\dotsc,y_t^K\} \sim \pi_t^\tau(\cdot \mid x)$, $\forall x \in \Dcal_0$, where $\pi_t^\tau$ is defined in \cref{eq:def_pitsmooth_app} with accommodating $\eta \to \etatd$.
    \State {\bf Rank responses:}\label{step:dno_rank_reg} For each $x \in \Dcal_0$, rank the corresponding $\{y_t^1,y_t^2,\dotsc,y_t^K,y^\gold\}$ using the pair-wise win-rate by sampling from the general preference function $\Pcal$.
    \State {\bf Filter and construct preference pairs:}\label{step:filtering_reg} Construct $\Dcal_t = \{(x,y_t^+,y_t^-)\}$, for all $x \in \Dcal_0$, and $(y_t^+,y_t^-)$ are large-margin pairs (based on the win-rate rank) within the responses for $x$ from the previous step.
    \State {\bf Contrastive learning:} Obtain $\pi_{t+1}$ by,
    \begin{gather*}
    \pi_{t+1} \leftarrow \argmax_{\pi \in \Pi} \E_{(x,y_t^+,y_t^-) \sim \Dcal_t} \log \left[ \sigma\left( \etatd \log\frac{\pi(y_t^+ \mid x)}{\widetilde\pi_t^\tau(y_t^+ \mid x)} - \etatd \log\frac{\pi(y_t^- \mid x)}{\widetilde\pi_t^\tau(y_t^- \mid x)} \right)\right],
    \end{gather*}
    where $\widetilde\pi_t^\tau(y\mid x) \coloneqq \pi_t(y \mid x)^{1 - \nicefrac{\tau}{\etatd}}\piref(y \mid x)^{\nicefrac{\tau}{\etatd}}, \forall (x,y)\in\Xcal\times\Ycal$ (the unnormalized version of $\pi_t^\tau(y\mid x)$ defined in \cref{eq:def_pitsmooth_app}, after accommodating $\eta \to \etatd$).
\EndFor
\State \Return best of $\pi_{1:(T+1)}$ on the validation data.
\end{algorithmic}
\end{algorithm}

Based on discuss above, we can then obtain the extension of \DNO to the regularized preferences in \cref{alg:dno_reg}, and its practical implementation in \cref{alg:NashDPO_reg}.
Note that, similar to \NashMD, the late-iteration option for both \cref{alg:dno_reg} and \cref{alg:NashDPO_reg} requires sampling from the smoothed policy $\pi_t^\tau$ (the mixture between $\pi_t$ and $\piref$, defined in \cref{eq:def_pitsmooth_app}). One solution to address this can be sampling from the token-level between $\pi_t$ and $\piref$ instead as suggested by \citet{munos2023nash}.

\section{Detailed Proofs}
\label{app:dno_proof}

In this section, we provide detailed proofs for our theoretical results.
Note that, the definitions and assumptions presented heavily adopts the ideas related to version space and concentrability from reinforcement learning theory literature \citep[esp.,][]{xie2021bellman,xie2023role}. Nevertheless, the descriptions provided herein are intentionally simplified to elucidate the core insights into the algorithmic design. A full and exhaustive theoretical analysis falls outside the primary scope of this paper.
We now make the following definitions and assumptions.
\begin{definition}[Feasible solution space]
\label{def:feasolspace}
For each iteration $t\in [T]$, we define $\Pi_t \subseteq \Pi$ as the feasible solution space for iteration $t$. The $\pi_t$ obtained by \cref{alg:dno} is always belong to $\Pi_t$, regardless of the randomness of the data sampling procedure in \cref{alg:dno}.
\end{definition}
Here, \cref{def:feasolspace} follows a similar spirit as the version space in RL theory literature, where $\Pi_t$ only contains policies that have a small empirical loss, which can be further converted to a small population loss under standard concentration procedures.
\begin{definition}[Concentrability coefficient over the feasible solution space]
\label{def:concentrability}
For all $t \in [T]$, suppose $\Pi_t$ is defined in \cref{def:feasolspace}, and $\mu_{1,t}$ and $\mu_{2,t}$ are some given data generate policy. Now, for any $t \in [T]$, we define $\Cfrak_t$ to be the concentrability coefficient at iteration $t$ over its feasible solution space, where
\begin{align*}
\frac{\E_{x \sim \rho, y_1 \sim \pi_{t+1}^\star(\cdot \mid x), y_2 \sim \pi_{t+1}(\cdot \mid x)}\left[ \left( \log\frac{\pi_{t+1}^\star(y_1 \mid x)}{\pi_{t+1}(y_1 \mid x)} - \log\frac{\pi_{t+1}^\star(y_2 \mid x)}{\pi_{t+1}(y_2 \mid x)} \right)^2\right]}{\E_{x \sim \rho, y_1 \sim \mu_{1,t}(\cdot \mid x), y_2 \sim \mu_{2,t}(\cdot \mid x)}\left[ \left( \log\frac{\pi_{t+1}^\star(y_1 \mid x)}{\pi_{t+1}(y_1 \mid x)} - \log\frac{\pi_{t+1}^\star(y_2 \mid x)}{\pi_{t+1}(y_2 \mid x)} \right)^2\right]} \leq \Cfrak_t,
\end{align*}
for any $\pi_{t+1} \in \Pi_{t+1}$ and any $\pi_{t+1}^\star \in \left\{\frac{1}{Z_\pi(x)} \pi(\cdot \mid x) \exp\left(\frac{r_\pi(x,\cdot)}{\eta}\right):\pi\in\Pi_t\right\}$; and here we use the definition of $r_\pi(x,y) \coloneqq \E_{y' \sim \pi(\cdot \mid x)}\left[\Pcal(y \succ y' \mid x)\right],\forall (x,y)\in\Xcal\times\Ycal$, and $Z_\pi(x) = \sum_{y \in \Ycal} \pi(y \mid x) \exp\left(\frac{r_\pi(x,y)}{\eta}\right),\forall x\in\Xcal$.
\end{definition}
\cref{def:concentrability} can be viewed as a natural extension of concentrability from the (offline) reinforcement learning literature to our setup.
\begin{assumption}[Realizability over the feasible solution space]
\label{asm:realizability}
For any $\pi \in \Pi_t$ where $\Pi_t$ is defined in \cref{def:feasolspace} for all $t \in [T]$, we assume the following soft-policy iteration update
\begin{align*}
\pi^{\sf SPI}(\cdot \mid x) \coloneqq \frac{1}{Z_\pi(x)} \pi(\cdot \mid x) \exp\left(\frac{r_\pi(x,\cdot)}{\eta}\right),
\end{align*}
where $r_\pi(x,y) \coloneqq \E_{y' \sim \pi(\cdot \mid x)}\left[\Pcal(y \succ y' \mid x)\right],\forall (x,y)\in\Xcal\times\Ycal$, and $Z_\pi(x) = \sum_{y \in \Ycal} \pi(y \mid x) \exp\left(\frac{r_\pi(x,y)}{\eta}\right),\forall x\in\Xcal$ is the partition function. 
\end{assumption}
\begin{assumption}[Boundedness over the feasible solution space]
\label{asm:boundedness}
Suppose $\Pi_t$ is defined in \cref{def:feasolspace} for all $t \in [T]$, then we assume $\log\frac{\pi(y \mid x)}{\pi_t(y\mid x)} \in [-\Rmax,\Rmax]$ for all $\pi \in \Pi$, $\pi_t \in \Pi_t$, and $(x,y)\in\Xcal\times\Ycal$.
\end{assumption}
\cref{asm:boundedness} may appear somewhat unconventional, as it explicitly assumes boundedness on the log probabilities. Nonetheless, it is important to note that the value of $\log\frac{\pi(y \mid x)}{\pi_t(y\mid x)}$ is directly measurable and controllable in practice, which is different from the common use case, such as maximum likelihood problems.

\begin{theorem}[Formal Version of \cref{thm:dno_guarantee_main}]
\label{thm:dno_guarantee_app}
Under \cref{asm:realizability,asm:boundedness}, and fix an arbitrary iteration $t \in [T]$. Suppose $\pi_{t+1}$ is from \cref{step:dno_learnpi} of \cref{alg:dno}, and $\pi_{t+1}^\star$ is defined in \cref{eq:def_pistar}. Then, we have
\begin{align*}
\E_{x \sim \rho}\left[ \left( \TV(\pi_{t+1}(\cdot \mid x),\pi_{t+1}^\star(\cdot \mid x)) \right)^2\right] \leq \Ocal \left( \frac{\Cfrak_t\Rmax^2\log(\nicefrac{|\Pi|}{\delta})}{N} \right),
\end{align*}
where $\Cfrak_t$ is defined in \cref{def:concentrability}.
\end{theorem}

\begin{proof}[\cpfname{thm:dno_guarantee_app}]
We will now present the proof using the following two-step procedure.

\textbf{Step 1: From regression with log loss to squared error bound.}\quad
By standard results on the regression with the logarithmic loss, we know,
\begin{footnotesize}
\begin{equation}
\label{eq:logloss_reg_bound}
\E_{(x,y_1,y_2) \sim \Dcal_t}\Big[ \sigma\big( \Delta_t^\star(x,y_1,y_2) \big) \log \left[ \sigma\big( \Delta_{\pi_{t+1},t}(x,y_1,y_2) \big) \right] + \sigma\big( \Delta_t^\star(x,y_2,y_1) \big) \log \left[ \sigma\big( \Delta_{\pi_{t+1},t}(x,y_2,y_1) \big) \right] \Big] \lesssim  \frac{\log(\nicefrac{|\Pi|}{\delta})}{N}.
\end{equation}
\end{footnotesize}

Note that similar results could also apply beyond finite $\Pi$. For simplicity, we omit the detailed discussion in our paper. For more in-depth discussions about regression with the logarithmic loss, the reader can refer to, e.g., \citet[][]{foster2021efficient}.

Next, by the Pinsker’s inequality, we have for any $z,\zhat \in [0,1]$,
\begin{align*}
\frac{(z - \zhat)^2}{2} \leq z \log\left( \frac{z}{\zhat} \right) + (1-z) \log\left( \frac{1-z}{1-\zhat} \right).
\end{align*}
Substituting the $z$ and $\zhat$ with \cref{eq:def_z} and combining with \cref{eq:logloss_reg_bound}, we obtain that
\begin{align}
\label{eq:reward_sq_Dt}
\E_{(x,y_1,y_2) \sim \Dcal_t}\left[ \big( \sigma \left( r_t(x,y_1) - r_t(x,y_2) \right) - \sigma \left( r_{\pi_{t+1},t}(x,y_1) - r_{\pi_{t+1},t}(x,y_2) \right)\big)^2\right] \lesssim \frac{\log(\nicefrac{|\Pi|}{\delta})}{N},
\end{align}
where $a \lesssim b$ means $a \leq c \cdot b$ for some absolute constant $c$. Then, by the standard concentration for squared loss, e.g., Lemma A.4 of \citet{xie2021bellman} with $\gamma = 0$, \cref{eq:reward_sq_Dt} implies
\begin{align}
\label{eq:reward_sq_pop}
\E_{(x,y_1,y_2) \sim \rho\times\mu_{1:2,t}}\left[ \big( \sigma \left( r_t(x,y_1) - r_t(x,y_2) \right) - \sigma \left( r_{\pi_{t+1},t}(x,y_1) - r_{\pi_{t+1},t}(x,y_2) \right)\big)^2\right] \lesssim \frac{\log(\nicefrac{|\Pi|}{\delta})}{N},
\end{align}
where we use ``$\times$'' as the shorthand of joint distribution for the sake of simplicity, for example, $(x,y_1,y_2) \sim \rho\times\mu_{1:2,t}$ is shorthand for $x \sim \rho, y_1 \sim \mu_{1,t}(\cdot \mid x), y_2 \sim \mu_{2,t}(\cdot \mid x)$.

By the definition of $r_t$ in \cref{step:def_rt} of \cref{alg:dno}, we know $r_t(x,y) \in [0,1]$ for all $(x,y)\in\Xcal\times\Ycal$. Thus, by a variant of mean value theorem, we know
\begin{equation}
\label{eq:mean_value_thm}
\begin{aligned}
&~ \big| r_t(x,y_1) - r_t(x,y_2)  -  r_{\pi_{t+1},t}(x,y_1) + r_{\pi_{t+1},t}(x,y_2) \big|
\\
\leq &~ \frac{\eta\Rmax}{1 - \sigma(1)} \big| \sigma \left( r_t(x,y_1) - r_t(x,y_2) \right) - \sigma \left( r_{\pi_{t+1},t}(x,y_1) - r_{\pi_{t+1},t}(x,y_2) \right)\big|,
\end{aligned}
\end{equation}
for any $(x,y_1,y_2) \in \Xcal\times\Ycal\times\Ycal$, where $\Rmax$ is introduced from \cref{asm:boundedness}.
This is because: let $a \coloneqq r_t(x,y_1) - r_t(x,y_2) \in [-1,1]$, and $b \coloneqq r_{\pi_{t+1},t}(x,y_1) - r_{\pi_{t+1},t}(x,y_2) \in [-\eta \Rmax, \eta \Rmax]$, and, then, we can directly verify that the slope we need to bound
$\nicefrac{\big|a  -  b \big|}{\big| \sigma \left( a \right) - \sigma \left( b \right)\big|}$ reaches its maximum at $a = 1$ and $b = \eta\Rmax$.

Combining \cref{eq:reward_sq_pop,eq:mean_value_thm}, we obtain
\begin{align}
\label{eq:dno_reward_bound}
\E_{(x,y_1,y_2) \sim \rho\times\mu_{1:2,t}}\left[ \big( r_t(x,y_1) - r_t(x,y_2) - r_{\pi_{t+1},t}(x,y_1) + r_{\pi_{t+1},t}(x,y_2)\big)^2\right] \lesssim \frac{\eta^2\Rmax^2\log(\nicefrac{|\Pi|}{\delta})}{N}.
\end{align}

\textbf{Step 2: Concentration in the policy space.}\quad
We now reason about the concentration of $\pi_{t+1} \to \pi_{t+1}^\star$ from \cref{eq:dno_reward_bound}, where $\pi_{t+1}^\star$ is defined in \cref{eq:def_pistar} and $\pi_{t+1}$ is the policy corresponding to the learned $r_{\pi_{t+1},t}$. By the definition of $r_{\pi,t}$ in \cref{eq:def_rpit}, we have
\begin{align*}
&~ r_t(x,y_1) - r_t(x,y_2) - r_{\pi_{t+1},t}(x,y_1) + r_{\pi_{t+1},t}(x,y_2)
\\
= &~ r_t(x,y_1) - r_t(x,y_2) - \eta \log\frac{\pi_{t+1}(y_1 \mid x)}{\pi_t(y_1 \mid x)} + \eta \log\frac{\pi_{t+1}(y_2 \mid x)}{\pi_t(y_2 \mid x)}
\\
= &~ \eta\log\frac{\pi_{t+1}^\star(y_1 \mid x)}{\pi_{t+1}(y_1 \mid x)} - \eta\log\frac{\pi_{t+1}^\star(y_2 \mid x)}{\pi_{t+1}(y_2 \mid x)}.
\end{align*}
This implies
\begin{gather}
\E_{(x,y_1,y_2) \sim \rho\times\mu_{1:2,t}}\left[ \left( \eta\log\frac{\pi_{t+1}^\star(y_1 \mid x)}{\pi_{t+1}(y_1 \mid x)} - \eta\log\frac{\pi_{t+1}^\star(y_2 \mid x)}{\pi_{t+1}(y_2 \mid x)} \right)^2\right] \lesssim \frac{\eta^2\Rmax^2\log(\nicefrac{|\Pi|}{\delta})}{N} \nonumber
\\
\label{eq:oneside_dno_bound}
\Longrightarrow \E_{(x,y_1,y_2) \sim \rho\times\pi_{t+1}^\star\times\pi_{t+1}}\left[ \left( \log\frac{\pi_{t+1}^\star(y_1 \mid x)}{\pi_{t+1}(y_1 \mid x)} - \log\frac{\pi_{t+1}^\star(y_2 \mid x)}{\pi_{t+1}(y_2 \mid x)} \right)^2\right] \lesssim \frac{\Cfrak_t\Rmax^2\log(\nicefrac{|\Pi|}{\delta})}{N},
\end{gather}
where the last step follows from the definition of $\Cfrak_t$ (\cref{def:concentrability}).

On the other hand, we have
\begin{align*}
&~ \E_{(x,y_1,y_2) \sim \rho\times\pi_{t+1}^\star\times\pi_{t+1}}\left[ \left(\log\frac{\pi_{t+1}^\star(y_1 \mid x)}{\pi_{t+1}(y_1 \mid x)} - \log\frac{\pi_{t+1}^\star(y_2 \mid x)}{\pi_{t+1}(y_2 \mid x)} \right)^2\right]
\\
= &~ \E_{(x,y_1,y_2) \sim \rho\times\pi_{t+1}^\star\times\pi_{t+1}}\left[ \left( \log\frac{\pi_{t+1}^\star(y_1 \mid x)}{\pi_{t+1}(y_1 \mid x)} \right)^2 + \left(\log\frac{\pi_{t+1}^\star(y_2 \mid x)}{\pi_{t+1}(y_2 \mid x)} \right)^2 - 2 \left( \log\frac{\pi_{t+1}^\star(y_1 \mid x)}{\pi_{t+1}(y_1 \mid x)} \right) \cdot \left(\log\frac{\pi_{t+1}^\star(y_2 \mid x)}{\pi_{t+1}(y_2 \mid x)} \right) \right]
\\
= &~ \E_{(x,y) \sim \rho\times\pi_{t+1}^\star}\left[ \left( \log\frac{\pi_{t+1}^\star(y \mid x)}{\pi_{t+1}(y \mid x)} \right)^2 \right]
+ \E_{(x,y) \sim \rho\times\pi_{t+1}}\left[\left(\log\frac{\pi_{t+1}^\star(y \mid x)}{\pi_{t+1}(y \mid x)} \right)^2 \right]
\\
&~ + 2 \E_{x \sim \rho}\Bigg[ \underbrace{\E_{y \sim \pi_{t+1}^\star(\cdot \mid x)}\left[ \log\frac{\pi_{t+1}^\star(y \mid x)}{\pi_{t+1}(y \mid x)} \right]}_{= \KL(\pi^\star_{t+1}(\cdot \mid x) \midd \pi_{t+1}(\cdot \mid x))} \cdot \underbrace{\E_{y \sim \pi_{t+1}(\cdot \mid x)}\left[ \log\frac{\pi_{t+1}(y \mid x)}{\pi_{t+1}^\star(y \mid x)} \right]}_{= \KL(\pi_{t+1}(\cdot \mid x) \midd \pi^\star_{t+1}(\cdot \mid x))} \Bigg]
\\
\label{eq:dno_mid}\numberthis
\geq &~ \E_{(x,y) \sim \rho\times\pi_{t+1}^\star}\left[ \left( \log\frac{\pi_{t+1}^\star(y \mid x)}{\pi_{t+1}(y \mid x)} \right)^2 \right]
+ \E_{(x,y) \sim \rho\times\pi_{t+1}}\left[\left(\log\frac{\pi_{t+1}^\star(y \mid x)}{\pi_{t+1}(y \mid x)} \right)^2 \right].
\end{align*}

Next, we fix an arbitrary $\xtilde \in \Xcal$, and we have
\begin{align}
&~ \E_{y \sim \pi_{t+1}^\star(\cdot \mid \xtilde)}\left[ \left( \log\frac{\pi_{t+1}^\star(y \mid \xtilde)}{\pi_{t+1}(y \mid \xtilde)} \right)^2 \right]
+ \E_{y \sim \pi_{t+1}(\cdot\mid \xtilde)}\left[\left(\log\frac{\pi_{t+1}^\star(y \mid \xtilde)}{\pi_{t+1}(y \mid \xtilde)} \right)^2 \right] \nonumber
\\
\geq &~ \left(\E_{y \sim \pi_{t+1}^\star(\cdot \mid \xtilde)}\left[ \left| \log\frac{\pi_{t+1}^\star(y \mid \xtilde)}{\pi_{t+1}(y \mid \xtilde)} \right| \right] \right)^2
+ \left( \E_{y \sim \pi_{t+1}(\cdot\mid \xtilde)}\left[\left| \log\frac{\pi_{t+1}^\star(y \mid \xtilde)}{\pi_{t+1}(y \mid \xtilde)} \right| \right] \right)^2
\tag{by Jensen's inequality}
\\
\label{eq:dno_fixx}
\gtrsim &~ \left(\E_{y \sim \pi_{t+1}^\star(\cdot \mid \xtilde)}\left[ \left| \log\frac{\pi_{t+1}^\star(y \mid \xtilde)}{\pi_{t+1}(y \mid \xtilde)} \right| \right] + \E_{y \sim \pi_{t+1}(\cdot\mid \xtilde)}\left[\left| \log\frac{\pi_{t+1}^\star(y \mid \xtilde)}{\pi_{t+1}(y \mid \xtilde)} \right| \right] \right)^2,
\end{align}
where $a \gtrsim b$ means $a \geq c \cdot b$ for some absolute constant $c$.

We now recall the definition of $f$-divergence: $D_f(p,q) \coloneqq \E_{y \sim q}[f(\nicefrac{p(y)}{q(y)})]$ for two distributions $p$ and $q$, where $f:\RR^+ \to \RR$ is convex with $f(1) = 0$. Thus, we can notice that,
\begin{align}
\label{eq:dno_df}
\E_{y \sim p}\left[ \left| \log\frac{p(y)}{q(y)} \right| \right] + \E_{y \sim q}\left[ \left| \log\frac{p(y)}{q(y)} \right| \right] = D_{f_1}(p,q), \quad \text{where}~f_1(u) \coloneqq (1 + u) \cdot \left|\log(u)\right|,~u\in\RR^+.
\end{align}
On the other hand, by the definition of total variation distance, we know
\begin{align}
\label{eq:tv_df}
\TV(p,q) = D_{f_2}(p,q), \quad \text{where}~f_2(u) \coloneqq \frac{1}{2} \left|u - 1\right|,~u\in\RR^+.
\end{align}
It is easy to verify that
\begin{align}
\label{eq:dnonorm_tv}
\text{$f_1(u) - f_2(u) = (1 + u) \cdot \left|\log(u)\right| - \frac{1}{2} \left|u - 1\right| \geq 0$, for all $u \in \RR^+$},
\end{align}
as it is a convex function on $\RR^+$ with a minimum of $0$ at $u=1$. Therefore, combining \cref{eq:dno_fixx,eq:dno_df,eq:tv_df,eq:dnonorm_tv}, we obtain
\begin{small}
\begin{align}
\label{eq:dno_perx_tv}
\E_{y \sim \pi_{t+1}^\star(\cdot \mid \xtilde)}\left[ \left( \log\frac{\pi_{t+1}^\star(y \mid \xtilde)}{\pi_{t+1}(y \mid \xtilde)} \right)^2 \right]
+ \E_{y \sim \pi_{t+1}(\cdot\mid \xtilde)}\left[\left(\log\frac{\pi_{t+1}^\star(y \mid \xtilde)}{\pi_{t+1}(y \mid \xtilde)} \right)^2 \right] \gtrsim \left[ \TV(\pi_{t+1}(\cdot \mid \xtilde),\pi_{t+1}^\star(\cdot \mid \xtilde)) \right]^2.
\end{align}
\end{small}
We then apply \cref{eq:dno_perx_tv} to \cref{eq:dno_mid},
\begin{align*}
&~ \E_{x \sim \rho}\left[ \left( \TV(\pi_{t+1}(\cdot \mid x),\pi_{t+1}^\star(\cdot \mid x)) \right)^2\right]
\\
\lesssim &~ \E_{(x,y_1,y_2) \sim \rho\times\pi_{t+1}^\star\times\pi_{t+1}}\left[ \left(\log\frac{\pi_{t+1}^\star(y_1 \mid x)}{\pi_{t+1}(y_1 \mid x)} - \log\frac{\pi_{t+1}^\star(y_2 \mid x)}{\pi_{t+1}(y_2 \mid x)} \right)^2\right]
\\
\lesssim &~ \frac{\Cfrak_t\Rmax^2\log(\nicefrac{|\Pi|}{\delta})}{N},
\end{align*}
where the last step follows from \cref{eq:oneside_dno_bound}. This completes the proof.
\end{proof}

\section{Additional Experimental Details}
\label{additional-experimental-details}

\textbf{Batched Prompting:} We also show in~\cref{fig:annotation_prompt} the prompt that we send to GPT-4 to annotate preferences. For the sake of efficiency, we ``batch'' requests to GPT-4, meaning that instead of sending every pair of candidate responses to be annotated, we show all candidates side-by-side and ask GPT-4 to apply the scoring rubric to each one in the context window. 

\begin{figure*}[ht]
\small
\begin{framed}
Review the user’s question and the corresponding response using the additive 6-point scoring system described below. Points are accumulated based on the satisfaction of each criterion:

\begin{enumerate}[nolistsep]
\item Give 1 point if the response is formatted as an answer, e.g. does not repeat itself over and over, is not a continuation of the question or other random question/answer pairs or examples, etc.
\item Add 1 more point if the response is relevant and provides some information related to the user’s inquiry, even if it is incomplete or contains some irrelevant content.
\item Add a third point if the response addresses a substantial portion of the user’s question without rambling, but does not completely resolve the query or provide a direct answer.
\item Award a fourth point if the response answers nearly all the elements of the user’s question correctly, in a useful and direct way, and is largely devoid of errors, extraneous content, or excess wordiness. 
\item Grant a fifth point if the response is entirely correct (without any errors or made-up facts), and directly answers all aspects of the user’s question in a helpful, well-organized, focused, and succinct style. 
\item Bestow a sixth point for a response that is entirely correct *AND* concise, impeccably tailored to the user’s question, without any unnecessary information or wordiness, reflecting expert knowledge, amounting to a superbly high-quality, focused, and insightful answer. 
\end{enumerate}

\#\#\# User: <instructions> 
\\

\#\#\# Candidate Responses:
\\

<candidate answers> 
\\

\#\#\# Instructions: 
After examining the user’s instruction and the response:
\begin{enumerate}
    \item Briefly justify your total score for EACH answer, up to 100 words.
    \item score EACH answer using the format: "Score: <totalpoints>"
\end{enumerate}

Remember to assess from the AI Assistant perspective, utilizing web search knowledge as necessary. 
To evaluate the response in alignment with this additive scoring model, we’ll systematically attribute points based on the outlined criteria.
Remember to score EACH candidate answer separately.
Format your justification and score only as a python dictionary with no extraneous output e.g. \{``id1\_justification'': ``...'', ``id1\_score'': 3, ``id2\_justification'' : ...\} with ``idX'' corresponding to the id of the X'th answer.

\{``id1\_justification'' :

\end{framed}
\caption{The annotation prompt submitted to GPT-4 to collect multiple side-by-side preference scores for \DNO. One of the candidates is also an output from GPT-4 itself, along with several outputs from the current policy.}
\label{fig:annotation_prompt}
\end{figure*}

\textbf{Cost Analysis}: We also do a brief cost analysis associated with the scaled-up experiment on 600k training inputs. The major line items are the cost of sampling outputs, annotating them with GPT-4 to construct training pairs, and then training the next iteration against those pairs. For \textit{each} of the six iterations:
\begin{enumerate}[nolistsep]
    \item Sampling: it took about 18-24 hours to inference 5 outputs for all 100k examples on 10 8xA100 80GB pods, depending on the average length, costing about \$6,000 based on spot pricing. 
    \item Annotation: the average number of prompt tokens sent to GPT-4 for annotation across iterations was about 450M, with an average of about 60M completion tokens, amounting to about \$34,000 based on the version of the endpoint we were using. 
    \item Training: ironically, training was the cheapest step, taking only 12-24 hours on two 8xA100 80GB nodes.
\end{enumerate}

\small
\begin{longtable}{|p{16cm}|}
\hline
{\color{Blue} Example 1} \\ \textbf{Question: When will the earth run out of fresh water?} \\ \hline
\begin{minipage}[t]{\linewidth}
{
{\color{Blue} Example 1: Initial SFT model output} \\
The earth will not run out of fresh water. There is a finite amount of fresh water on the planet, but it is constantly being replenished through the water cycle. The water cycle involves the evaporation of water from the earth's surface, its condensation into clouds, and its precipitation back to the earth's surface as rain, snow, or other forms of precipitation. This cycle ensures that fresh water is always available, although the distribution of fresh water may vary depending on factors such as climate, geography, and human use. \\
}
\end{minipage} \\ \hline
\begin{minipage}[t]{\linewidth}
{
{\color{Blue} Example 1: \DNO Iter 1} \\
It is difficult to predict exactly when the Earth will run out of fresh water, as it depends on various factors such as climate change, population growth, water usage, and technological advancements in water management and conservation. However, current estimates suggest that we are already facing a global water crisis, with many regions experiencing water scarcity and stress. \\

According to the United Nations, about 2.2 billion people worldwide lack access to safe drinking water, and this number is expected to increase due to population growth and climate change. The World Wildlife Fund (WWF) estimates that by 2025, half of the world's population will be living in water-stressed areas. \\

While the Earth's total water supply remains constant, the availability of fresh water is decreasing due to overuse, pollution, and climate change. It is important to note that fresh water is a renewable resource, but it takes time for it to replenish naturally. Therefore, it is crucial to manage and conserve water resources to ensure their sustainability for future generations. \\

To address the global water crisis, it is necessary to implement sustainable water management practices, improve water infrastructure, promote water conservation and efficiency, and raise awareness about the importance of water conservation. \\

In summary, while the Earth will not run out of water, the availability of fresh water is becoming increasingly scarce in many regions. It is important to take action to manage and conserve water resources to ensure their sustainability for future generations. \\

It is also important to note that the availability of fresh water varies depending on the location and the specific water sources. Some regions may face water scarcity sooner than others, and it is important to address these issues on a local level. \\

In conclusion, while it is difficult to predict exactly when the Earth will run out of fresh water, it is clear that we are already facing a global water crisis. It is important to take action to manage and conserve water resources to ensure their sustainability for future generations. \\

Reference:

United Nations. (2019). World Water Development Report 2019. Retrieved from https://www.unwater.org/publications/world-water-development-report-2019/

World Wildlife Fund. (2019). Water Scarcity. Retrieved from https://www.worldwildlife.org/issues/water-scarcity \\

Note: The information provided in this response is based on current scientific knowledge and available data. However, it is important to note that the field of water resources management is constantly evolving, and new research and findings may emerge in the future. \\

Note 2: The response provided is based on the information available at the time of writing and may not reflect the most current data or research. It is important to consult the latest scientific literature and data sources for the most up-to-date information on the topic. \\

Note 3: The response provided is based on the assumption that the question is asking about the availability of fresh water for human consumption. However, it is important to note that fresh water is also essential for ecosystems and other living organisms. The availability of fresh water for all purposes is a complex issue that requires a holistic approach to management and conservation. \\

Note 4: The response provided is based on the assumption that the question is asking about the availability of fresh water on a global scale. \\
}
\end{minipage} \\ \hline
\begin{minipage}[t]{\linewidth}
{
{\color{Blue} Example 1: \DNO Iter 2} \\
It is not accurate to say that the Earth will "run out" of fresh water, but rather that we could potentially deplete or contaminate fresh water resources to a point where they can no longer sustain us. Fresh water is a renewable resource, but its availability depends on various factors such as climate, geography, and human activities. Here are some key points to consider: \\

1. Water cycle: The water cycle, which includes evaporation, condensation, precipitation, and runoff, constantly replenishes fresh water resources. However, the rate of replenishment varies depending on the region and climate.

2. Groundwater: Groundwater is a significant source of fresh water, but it is not infinite. Over-extraction of groundwater can lead to depletion and contamination. It takes thousands of years for groundwater to replenish naturally.

3. Surface water: Surface water, such as rivers and lakes, is also a renewable resource, but it can be affected by climate change, deforestation, and other human activities. Droughts and changes in precipitation patterns can reduce the availability of surface water.

4. Contamination: Contamination from pollution, agricultural runoff, and other sources can render fresh water unusable. This includes both groundwater and surface water.

5. Population growth: As the global population continues to grow, the demand for fresh water is increasing. This can lead to over-extraction, contamination, and other issues that can impact the availability of fresh water.

6. Climate change: Climate change is expected to exacerbate water scarcity in many regions due to changes in precipitation patterns, increased evaporation, and melting of glaciers and snowpacks. \\

In conclusion, while fresh water is a renewable resource, its availability is not infinite and can be impacted by various factors. It is essential to manage and conserve fresh water resources sustainably to ensure their long-term availability for future generations.

So, instead of asking when the Earth will "run out" of fresh water, it is more accurate to ask how we can manage and conserve fresh water resources to ensure their long-term sustainability. \\
}
\end{minipage} \\ \hline
\begin{minipage}[t]{\linewidth}
{
{\color{Blue} Example 1: \DNO Iter 3} \\
The Earth as a planet does not "run out" of fresh water, as it continues to cycle through the hydrological cycle, which includes processes such as evaporation, condensation, precipitation, and infiltration. However, the availability and accessibility of fresh water for human use are significant concerns. \\
\\
The freshwater crisis is primarily driven by population growth, urbanization, agricultural expansion, and climate change, which lead to increased demand, pollution, and variability in water availability. According to the United Nations, approximately 2.2 billion people lack access to safe drinking water services, and 4.2 billion people experience severe water scarcity for at least one month a year. \\
\\
While the Earth will not run out of fresh water, the distribution and availability of fresh water for human use are becoming increasingly challenging. To address this issue, it is crucial to implement sustainable water management practices, reduce water waste, and promote equitable access to fresh water resources. \\
\\
In summary, the Earth will not run out of fresh water, but the availability and accessibility of fresh water for human use are becoming increasingly limited, and it is essential to address this challenge through sustainable water management practices and equitable access to resources. \\
}
\end{minipage} \\ \hline
\caption{Outputs of the various \DNO models across iterations for the question: {\color{Blue} When will the earth run out of fresh water?}. The output for Iter-1 is a bit too long as shown in ~\cref{tab:count-of-training-pairs}. We believe this could be addressed with better hyperparameter tuning or preference data. We fine the initial SFT model's response to be missing important points about how the premise of the question is best addressed by highlighting access to freshwater. The last reponse, for Iter 3 is more informative and specific than the initial SFT response. }
\label{ex:example-1}
\end{longtable}

\small
\begin{longtable}{|p{16cm}|}
\hline
\begin{minipage}[t]{\linewidth}
{
{\color{Green} Example 2} \\ 
\textbf{Question: what data structure would you use to design a url lookup system to quickly determine whether a url exists, or if not, its closest prefix. Explain.}
}
\end{minipage} \\ \hline
\begin{minipage}[t]{\linewidth}
{
{\color{Green} Example 2: Initial SFT model output} \\
To design a URL lookup system that can quickly determine whether a URL exists or find its closest prefix, we can use a Trie data structure. \\

A Trie is a tree-like data structure that stores strings in a hierarchical manner. Each node in the Trie represents a character in the string, and the edges represent the characters themselves. The root node represents an empty string, and each subsequent node represents a character appended to the previous node's string. \\

In our URL lookup system, we can use a Trie to store all the existing URLs. Each node in the Trie will represent a URL prefix, and the edges will represent the characters in the URL. \\

To determine whether a URL exists, we can traverse the Trie starting from the root node. If we reach a node that represents a complete URL, we know that the URL exists. If we reach a node that represents an incomplete URL, we know that the URL does not exist. \\

To find the closest prefix of a URL, we can traverse the Trie starting from the root node. We can keep track of the longest prefix we have encountered so far. Whenever we encounter a node that represents a URL prefix, we compare its length with the length of the longest prefix we have encountered so far. If the length of the current prefix is greater, we update the longest prefix. \\

By using a Trie, we can efficiently determine whether a URL exists or find its closest prefix in constant time complexity. This is because the Trie allows us to traverse the data structure in a hierarchical manner, and we can quickly determine whether a URL exists or find its closest prefix by following the appropriate edges. \\
}
\end{minipage} \\ \hline
\begin{minipage}[t]{\linewidth}
{
{\color{Green} Example 2: \DNO Iter 1} \\
To design a URL lookup system that quickly determines whether a URL exists or finds its closest prefix, you can use a combination of a Trie (Prefix Tree) and a Bloom Filter. This approach provides efficient lookup and prefix matching while maintaining a low memory footprint.
\begin{enumerate}
    \item Trie (Prefix Tree): A Trie is a tree-like data structure that stores strings as nodes. Each node in the Trie represents a character in the string, and the edges represent the characters themselves. The root node is empty, and each subsequent node contains a single character. This data structure allows for efficient prefix matching and string lookup.
    \item Bloom Filter: A Bloom Filter is a probabilistic data structure that uses a bit array and a set of hash functions to determine whether an element is a member of a set. It provides a fast way to check if a URL exists in the system, but it may produce false positives.
\end{enumerate}

Here's how you can use these data structures together: \\

\begin{enumerate}
    \item Create a Trie to store all the URLs in the system. Each URL is added as a node in the Trie, and the edges represent the characters in the URL.
    \item Create a Bloom Filter to store the URLs in the system. The Bloom Filter will have a fixed size and a set of hash functions. When a URL is added to the system, it is hashed using the hash functions, and the resulting bits are set in the Bloom Filter.
    \item To check if a URL exists, first check the Bloom Filter. If the URL is not present in the Bloom Filter, it does not exist in the system. If the URL is present in the Bloom Filter, look up the URL in the Trie. If the URL is found in the Trie, it exists in the system. If the URL is not found in the Trie, it may be a prefix of an existing URL.
    \item To find the closest prefix, look up the URL in the Trie. If the URL is not found, start from the root node and traverse the Trie until you find the longest prefix that matches the input URL. This will give you the closest prefix.
\end{enumerate}

This approach provides efficient lookup and prefix matching while maintaining a low memory footprint. The Trie allows for efficient string lookup, and the Bloom Filter provides a fast way to check if a URL exists in the system. However, the Bloom Filter may produce false positives, so it should be used in conjunction with the Trie to ensure accuracy. \\

Note: The Trie and Bloom Filter can be optimized further by using a compressed Trie or a Bloom Filter with a variable-size bit array. These optimizations can reduce the memory footprint and improve the performance of the system. \\
}
\end{minipage} \\ \hline
\begin{minipage}[t]{\linewidth}
{
{\color{Green} Example 2: \DNO Iter 2} \\
To design a URL lookup system that can quickly determine if a URL exists or find its closest prefix, you can use a data structure called a Trie (also known as a Prefix Tree). A Trie is a tree-like data structure that stores strings as nodes, where each edge is labeled with a single character. This data structure is particularly efficient for searching and inserting strings with common prefixes.

Here's how you can use a Trie to implement a URL lookup system:
\begin{enumerate}
    \item  Create a Trie node structure with properties like `is\_end\_of\_word` (indicates if the node represents a complete URL) and a dictionary to store child nodes.
    \item Insert URLs into the Trie by iterating over each character in the URL and adding it as a child node to the current node. If the character is not present in the current node's dictionary, create a new node and add it.
    \item To search for a URL, start at the root node and iterate over each character in the URL. If the character is present in the current node's dictionary, move to the corresponding child node. If the character is not present, return the closest prefix by traversing the Trie upwards from the last valid node.
\end{enumerate}

Here's an example implementation in Python:
\begin{python}
class TrieNode:
    def __init__(self):
        self.is_end_of_word = False
        self.children = {}

class Trie:
    def __init__(self):
        self.root = TrieNode()

    def insert(self, url):
        current_node = self.root
        for char in url:
            if char not in current_node.children:
                current_node.children[char] = TrieNode()
            current_node = current_node.children[char]
        current_node.is_end_of_word = True

    def search(self, url):
        current_node = self.root
        for char in url:
            if char in current_node.children:
                current_node = current_node.children[char]
            else:
                return self.find_closest_prefix(current_node)
        return url if current_node.is_end_of_word else None

    def find_closest_prefix(self, current_node):
        prefix = []
        while current_node:
            prefix.append(list(current_node.children.keys())[0])
            current_node = current_node.children[list(current_node.children.keys())[0]]]
        return ''.join(reversed(prefix))[::-1]
\end{python}

This implementation allows you to quickly determine if a URL exists or find its closest prefix by searching the Trie. The `insert` method inserts URLs into the Trie, and the `search` method searches for a URL or returns its closest prefix if it doesn't exist. \\

}
\end{minipage} \\ \hline
\begin{minipage}[t]{\linewidth}
{
{\color{Green} Example 2: \DNO Iter 3} \\
To design a URL lookup system that can quickly determine if a URL exists or find its closest prefix, you can use a data structure called the Trie (also known as Prefix Tree). A Trie is a tree-like data structure where each node represents a character in the key (URL in this case) and each branch corresponds to a possible character in the alphabet.

Here's how you can implement a URL lookup system using a Trie:

\begin{enumerate}
    \item Create a Trie node with an end-of-string (EOS) marker. This node will represent the end of a URL and will indicate that the URL exists.
    \item For each URL in your dataset, insert it into the Trie. While inserting, traverse the Trie by following the corresponding characters of the URL. If a character doesn't exist in the Trie, create a new node for it.
    \item When inserting a URL, follow these steps:
    \begin{enumerate}
        \item Traverse the Trie using the URL's characters.
        \item If a node doesn't have a child for the next character, create a new node for it.
        \item If a node doesn't have a child for the next character, create a new node for it.
        \item Repeat step a and b until you reach the end of the URL.
        \item Set the EOS marker for the last node you visited.
    \end{enumerate}
\end{enumerate}

To search for a URL or find its closest prefix, follow these steps:

\begin{enumerate}[nolistsep]
    \item Traverse the Trie using the characters of the URL you want to search.
    \item If a node doesn't have a child for the current character, stop the search. This URL doesn't exist in the Trie.
    \item If you reach the EOS marker, the URL exists in the Trie.
    \item If you stop the search before reaching the EOS marker, you've found the closest prefix of the URL.
\end{enumerate}~
}
\end{minipage} \\ \hline
\caption{Outputs of the various \DNO models across iterations for an interview question relating to the {\color{Green} design of a URL lookup system.}. The last response from Iter 3 is more informative and clearer, and doesn't contain misleading information (searching a trie runs in linear complexity in the length of the strings, not constant...). The response for Iter 1 contains an implementation of a Trie, which may be unnecessary because the user did not ask for it.}
\label{ex:example-2}
\end{longtable}

\small
\begin{longtable}{|p{16cm}|}
\hline
\begin{minipage}[t]{\linewidth}
{
{\color{Red} Example 3} \\ 
\textbf{Question: What factors led to the outbreak of ww1}
}
\end{minipage} \\ \hline
\begin{minipage}[t]{\linewidth}
{
{\color{Red} Example 3: Initial SFT model output} \\
There were several factors that led to the outbreak of World War I:

1. Nationalism: The rise of nationalism in Europe led to a desire for territorial expansion and the desire to unite people of the same ethnicity or language under one nation. This led to tensions between countries and a desire to assert dominance over neighboring territories.

2. Imperialism: The scramble for colonies in Africa and Asia led to competition between European powers, which further fueled tensions and rivalries.

3. Militarism: The arms race between European powers led to a buildup of military forces and weapons, which created a sense of insecurity and fear among neighboring countries.

4. Alliances: The complex system of alliances between European powers created a web of interdependencies that made it difficult for any one country to remain neutral in the event of a conflict.

5. Assassination of Archduke Franz Ferdinand: The assassination of Archduke Franz Ferdinand of Austria-Hungary by a Serbian nationalist in 1914 was the immediate trigger for the war, but it was the culmination of these other factors that led to the outbreak of the war.
}\\
\end{minipage} \\ \hline
\begin{minipage}[t]{\linewidth}
{
{\color{Red} Example 3: \DNOmore Iter 1} \\
The outbreak of World War I was the result of a complex interplay of factors, including:

\begin{enumerate}
    \item Imperialism and Colonialism: The competition for colonies and territories among European powers led to tensions and rivalries. The scramble for Africa and the race to acquire overseas territories created a sense of competition and mistrust among nations.
    \item Militarism and Arms Race: The arms race among European powers, particularly between Germany and Britain, led to a buildup of military power. The development of new technologies, such as machine guns, artillery, and submarines, increased the destructive potential of warfare.
    \item Alliances and Blocs: The complex web of alliances and blocs among European powers created a situation where conflicts in one region could escalate into a larger war. The Triple Entente (France, Russia, and Britain) and the Triple Alliance (Germany, Austria-Hungary, and Italy) were the main alliances that formed the basis of the conflict.
    \item Nationalism and Ethnic Tensions: The rise of nationalism and ethnic tensions in various regions, such as the Balkans, contributed to the outbreak of war. The desire for self-determination and independence led to conflicts and tensions among different ethnic groups.
    \item Assassination of Archduke Franz Ferdinand: The assassination of Archduke Franz Ferdinand of Austria-Hungary by a Serbian nationalist in 1914 was the immediate trigger for the war. The complex web of alliances and rivalries among European powers led to a series of diplomatic maneuvers and ultimatums, which ultimately resulted in the declaration of war.
    \item Economic Factors: Economic competition and trade disputes among European powers also contributed to the tensions leading up to the war. The economic rivalry between Germany and Britain, particularly in the context of colonial competition, added to the existing tensions.
    \item Political Instability: Political instability and power struggles within various countries, such as the assassination of King Alexander I of Serbia in 1903, created a sense of uncertainty and instability. These political upheavals contributed to the overall atmosphere of tension and conflict.
    \item Failure of Diplomacy: The failure of diplomatic efforts to resolve conflicts and tensions among European powers led to the breakdown of communication and trust. The inability of leaders to find peaceful solutions to disputes contributed to the escalation of tensions and the eventual outbreak of war.
    \item Miscalculation and Miscommunication: Miscalculations and miscommunications among military and political leaders also played a role in the outbreak of war. The lack of clear communication and misunderstandings between different countries led to a series of missteps and escalations.
    \item Systemic Factors: The complex system of alliances, blocs, and rivalries among European powers created a situation where conflicts in one region could have a domino effect on other regions. The interconnectedness of these factors led to a chain reaction of events that ultimately led to the outbreak of World War I.
\end{enumerate}~

}
\end{minipage} \\ \hline
\begin{minipage}[t]{\linewidth}
{
{\color{Red} Example 3: \DNOmore Iter 6} \\
The outbreak of World War 1, which started on July 28, 1914, resulted from a complex web of events and circumstances that had been developing for years prior. Here are some of the key factors that led to the war:
\begin{enumerate}
    \item **Militarism:** All major powers had been engaged in an arms race, building up their military forces and fortifications along their borders. This led to a sense of insecurity and a belief that preventive war was necessary to secure national interests.
    \item **Alliance System:** Europe was divided into two major alliances: the Triple Entente (France, Russia, and Britain) and the Triple Alliance (Germany, Austria-Hungary, and Italy). These alliances created a complex web of obligations that meant if one member was attacked, all members were obliged to come to their aid.
    \item  **Imperialism:** The scramble for colonies had led to territorial disputes and competition for resources, leading to tensions between the major powers.
    \item **Nationalism:** Strong feelings of national identity and pride led to a disregard for other nations' interests. The assassination of Archduke Franz Ferdinand of Austria, which triggered the war, was a result of nationalist tensions in the Balkans.
    \item **Security Dilemma:** The actions of one country led other countries to take defensive measures, which in turn made the first country feel more insecure, leading to a spiral of actions and reactions that escalated the situation.
    \item **Assassination of Archduke Franz Ferdinand:** The assassination by Gavrilo Princip, a member of the Black Hand, a Serbian nationalist group, provided the immediate trigger for the war.
    \item **Ultimatum to Serbia:** Austria-Hungary delivered an ultimatum to Serbia, which Serbia did not accept fully, leading to Austria-Hungary declaring war on Serbia.
    
\end{enumerate}

These factors combined to create a powder keg that was waiting for a spark to ignite the war. Once the spark was lit, the complex web of alliances meant that the conflict quickly escalated into a global war.\\
}
\end{minipage} \\ \hline
\caption{Outputs of first and last \DNOmore iterations for the third example {\color{Red} What factors led to the outbreak of ww1?}. The last response from Iter 6 has a higher information density; it recalls more key facts and entities, including the specific date of the start of WW1, Triple Entente, the Triple Alliance, Gavrilo Princip, and the Black Hand. Iteration 1 also contains some of this information, but is too wordy. The initial SFT model seems a lot more superficial.}
\label{ex:example-3}
\end{longtable}

\end{document}